\documentclass[10pt,
final,
 journal, 
 twocolumn,
compsoc,
 letterpaper]{IEEEtran} 

\ifCLASSOPTIONcompsoc
  \usepackage[nocompress]{cite}
\else
  \usepackage{cite}
\fi

\usepackage[usenames, dvipsnames]{color}
\usepackage{standalone}
\usepackage[inline]{enumitem}
\usepackage[utf8]{inputenc} 
\usepackage[T1]{fontenc}    
\usepackage{hyperref}       
\usepackage{url}            
\usepackage{booktabs}       
\usepackage{amsfonts}       
\usepackage{amsthm}
\usepackage{nicefrac}       
\usepackage{microtype}      

\usepackage{srcltx}
\usepackage{multirow}
\usepackage{amsmath}
\usepackage{amssymb}
\usepackage{amsfonts}
\usepackage{graphics}
\usepackage{tikz} 
\usetikzlibrary{fit,positioning}
\usetikzlibrary{arrows,positioning,automata}
\usepackage{psfrag}
\usepackage{epsfig}
\usepackage{subfigure}
\usepackage{array}
\usepackage{algorithm}
\usepackage{algorithmic}
\usepackage{pifont}
\usepackage{srcltx}
\usepackage{makeidx}  
\usepackage{bbm}
\usepackage{hhline}
\usepackage{yfonts}
\usepackage{color}
\usepackage{url}
\usepackage{dsfont}
\usepackage{caption}
\usepackage{balance} 

\newtheorem{thm}{Theorem}
\newtheorem{cor}{Corollary}

\newtheorem{lem}{Lemma}

\newtheorem{assumption}{Assumption}
\newcommand{\figurewidth}{2.9in}
\newcommand{\tilefigwidth}{2.1in}
\long\def\comment#1{}
\usepackage{etoolbox}
\title{Norm-Preservation: Why Residual Networks Can Become Extremely Deep?}

\author{Alireza~Zaeemzadeh,~\IEEEmembership{Student Member,~IEEE,}
        Nazanin~Rahnavard,~\IEEEmembership{Senior Member,~IEEE,}
        and~Mubarak~Shah,~\IEEEmembership{Fellow,~IEEE}
\thanks{
The authors are with the School of Electrical Engineering and Computer Science, University of Central Florida, Orlando, FL 32816 USA (e-mail: zaeemzadeh@eecs.ucf.edu; nazanin@eecs.ucf.edu; shah@crcv.ucf.edu).
}
}
\IEEEtitleabstractindextext{

\begin{abstract}
Augmenting neural networks with skip connections, as introduced in the so-called ResNet architecture, surprised the community by enabling the training of networks of more than 1,000 layers with significant performance gains. This paper deciphers ResNet by analyzing the effect of skip connections, and puts forward new theoretical results on the advantages of identity skip connections in neural networks. We prove that the skip connections in the residual blocks facilitate preserving the norm of the gradient, and lead to stable back-propagation, which is desirable from optimization perspective. We also show that, perhaps surprisingly, as more residual blocks are stacked, the norm-preservation of the network is enhanced. Our theoretical arguments are supported by extensive empirical evidence.

Can we push for extra norm-preservation? We answer this question by proposing an efficient method to regularize the singular values of the convolution operator and making the ResNet's transition layers extra norm-preserving. Our numerical investigations demonstrate that the learning dynamics and the classification performance of ResNet can be improved by making it even more norm preserving. Our results and the introduced modification for ResNet, referred to as Procrustes ResNets, can be used as a guide for training deeper networks and can also inspire new deeper architectures.

\end{abstract}
\begin{IEEEkeywords}
Residual Networks, Convolutional Neural Networks, Optimization Stability, Norm Preservation, Spectral Regularization.
\end{IEEEkeywords}
}
\begin{document}
\maketitle

\section{Introduction}
\label{sec:intro}
 Deep neural networks have progressed rapidly during the last few years, achieving outstanding, sometimes super human, performance \cite{Silver2017MasteringKnowledge}. It is known that the depth of the network, i.e., number of stacked layers, is of decisive significance. It is shown that as the networks become deeper, they are capable of representing more complex mappings \cite{Montufar2014OnNetworks}. However, deeper networks are notoriously harder to train. As the number of layers is increased, optimization issues arise and, in particular,  avoiding vanishing/exploding gradients is essential to optimization stability of such networks. Batch normalization, regularization, and initialization techniques have shown to be useful remedies for this problem \cite{He2015DelvingClassification,Ioffe2015BatchShift}. \par 

Furthermore, it has been observed that as the networks become increasingly deep, the performance gets saturated or even deteriorates \cite{He2016DeepRecognition}. This problem has been addressed by many recent network designs \cite{He2016IdentityNetworks,He2016DeepRecognition,Srivastava2015TrainingNetworks,Huang2017DenselyNetworks}. All of these approaches use the same design principle: skip connections. This simple trick makes the information flow across the layers easier, by bypassing the activations from one layer to the next using skip connections. Highway Networks \cite{Srivastava2015TrainingNetworks}, ResNets \cite{He2016IdentityNetworks,He2016DeepRecognition}, and DenseNets \cite{Huang2017DenselyNetworks} have consistently achieved state-of-the-art performances by using skip connections in different network topologies. The main goal of skip connection is to enable the information to flow through many layers without attenuation. In all of these efforts, it is observed empirically that it is crucial to keep the information path \emph{clean} by using identity mapping in the skip connection. It is also observed that more complicated transformations in the skip connection lead to more difficulty in optimization, even though such transformations have more representational capabilities \cite{He2016IdentityNetworks}. 
This observation implies that \emph{identity} skip connection, while provides adequate representational ability, has a great feature of optimization stability, enabling deeper well-behaved networks.\par

Since the introduction of Residual Networks (ResNets)~\cite{He2016IdentityNetworks,He2016DeepRecognition}, there have been some efforts on understanding how the residual blocks may help the optimization process and how they improve the representational ability of the networks. Authors in \cite{Orhan2018SkipSingularities} showed that skip connection eliminates the singularities caused by the model non-identifiability. This makes the optimization of deeper networks feasible and faster. Similarly, to understand the optimization landscape of ResNets, authors in \cite{Hardt2017IdentityLearning} prove that linear residual networks have no critical points other than the global minimum. This is in contrast to plain linear networks, in which other critical points may exist \cite{Kawaguchi2016DeepMinima}. Furthermore, authors in \cite{Balduzzi2017TheQuestion} show that as depth increases, gradients of plain networks resemble white noise and become less correlated. This phenomenon, which is referred to as {\em shattered gradient} problem, makes training more difficult. Then, it is demonstrated that residual networks reduce shattering, compared to plain networks, leading to numerical stability and easier optimization. \par  

In this paper, we present and analytically study another desirable effect of identity skip connection:  \emph{the norm preservation of error gradient}, as it propagates in the backward path. We show theoretically and empirically that each residual block in ResNets is \emph{increasingly norm-preserving}, as the network becomes \emph{deeper}. 
This interesting result is in contrast to hypothesis provided in \cite{Veit2016ResidualNetworks}, which states that residual networks avoid vanishing gradient \emph{solely} by shortening the effective path of the gradient. \par 

Furthermore, we show that identity skip connection enforces the norm-preservation during the training, leading to well-conditioning and easier training. This is in contrast to  the initialization techniques, in which the initialization distribution is modified to make the training easier \cite{He2015DelvingClassification,Glorot2010UnderstandingNetworks}. This is done by keeping the variance of weights gradient the same across layers. However, as observed in \cite{Glorot2010UnderstandingNetworks} and verified by our experiments, using such initialization methods, although the network is initially fairly norm-preserving, the norms of the gradients diverge as training progresses.  \par 

We analyze the role of identity mapping as skip connection in the ResNet architecture from a theoretical perspective. Moreover, we use the insight gained from our theoretical analysis to propose modifications to some of the building blocks of the ResNet architecture. Two main contributions of this paper are as follows. \par 

\begin{itemize}
\item \textbf{Proof of the Norm Preservation of ResNets}: We show that having identity mapping in the shortcut path leads to norm-preserving building blocks. Specifically, identity mapping shifts all the singular values of the transformations towards $1$. This makes the optimization of the network much easier by preserving the magnitude of the gradient across the layers. Furthermore, we show that, perhaps surprisingly, {\em as the network becomes deeper, its building blocks become more norm-preserving}. 
Hence, the gradients can flow smoothly through very deep networks, making it possible to train such networks. Our experiments validate our theoretical findings.
\item \textbf{Enhancing Norm Preservation}: Using insights from  our theoretical investigation, we propose important modifications to the \emph{transition blocks} in the ResNet architecture. The transition blocks are used to change the number of channels and feature map size of the activations. 
Since these blocks do not use identity mapping as the skip connection, in general, they  do not preserve the norm of the gradient.  We propose to change the dimension of the activations in a norm preserving manner, such that the network becomes even more norm-preserving. For that, we propose a computationally efficient method to set the nonzero singular values of the convolution operator, without using singular value decomposition. We refer to the proposed architecture as Procrustes ResNet (ProcResNet). Our experiments demonstrate that the proposed norm-preserving blocks are able to improve the optimization stability and performance of ResNets. 
\end{itemize}

The rest of the paper is organized as follows. In Section \ref{sec:study}, the theoretical results and the bounds for norm-preservation of linear and nonlinear residual networks are presented. Then, in Section \ref{sec:Transition Block}, we show how to enhance the norm preservation of the residual networks by introducing a new computationally efficient regularization of convolutions. To verify our theoretical investigation and to demonstrate the effectiveness of the proposed regularization, we provide our experiments in Section \ref{sec:experiments}. Finally, Section \ref{sec:conclusions} draws conclusions.

\section{Norm-Preservation of Residual Networks}
\label{sec:study}

Our following main theorem states that, under certain conditions, a deep residual network representing a nonlinear mapping is norm-preserving in the backward path. We show that, at each residual block, the norm of the gradient with respect to the input is close to the norm of gradient with respect to the output. In other words, \emph{the residual block with identity mapping, as the skip connection, preserves the norm of the gradient in the backward path}. 
This results in several useful characteristics such as avoiding vanishing/exploding gradient, stable optimization, and performance gain. 

Suppose we want to represent a nonlinear mapping $\mathcal{F}: \mathbb{R}^N \rightarrow \mathbb{R}^N$ with a sequence of $L$ non-linear residual blocks of form:
\begin{equation}
\label{eq:nonlinear-residual-mapping}
\boldsymbol{x}_{l+1} = \boldsymbol{x}_{l} + F_l(\boldsymbol{x}_{l}).
\end{equation}
As illustrated in Figure \ref{subfig:resnet_orig_block}, $\boldsymbol{x}_{l}$ and $\boldsymbol{x}_{l+1}$ represent  respectively the input and output at $l^{\text{th}}$ layer. $F_l(\boldsymbol{x}_{l})$ is the residual transformation learned by the $l^{\text{th}}$ layer.
Before presenting the theorem, we lay out the following assumptions on $\mathcal{F}$.
\begin{assumption}
\label{assump:nonlinear_map}
The function $\mathcal{F}: \mathbb{R}^N \rightarrow \mathbb{R}^N$ is differentiable, invertible, and satisfies the following conditions:
\begin{enumerate}[label=(\roman*)]
    \item \label{assump:smooth} $\forall \boldsymbol{x},\boldsymbol{y},\boldsymbol{z} $ with bounded norm, $\exists \alpha > 0$, $\|  (\mathcal{F}'(\boldsymbol{x}) - \mathcal{F}'(\boldsymbol{y}))\boldsymbol{z} \| \leq \alpha \|\boldsymbol{x} - \boldsymbol{y}\| \|\boldsymbol{z}\|$,
    \item \label{assump:inv_lip} $\forall \boldsymbol{x},\boldsymbol{y} $ with bounded norm, $\exists \beta > 0$, $\| \mathcal{F}^{-1}(\boldsymbol{x}) - \mathcal{F}^{-1}(\boldsymbol{y}) \| \leq \beta \| \boldsymbol{x} - \boldsymbol{y} \|$, and
    \item \label{assump:det} $\exists \boldsymbol{x}$ with bounded norm such that $Det(\mathcal{F}'(\boldsymbol{x})) > 0$,
\end{enumerate}
\end{assumption}
$\alpha$ and $\beta$ are constants, independent of network size and architecture. 
Also, we assume that the domain of inputs is bounded. 
By rescaling inputs, we can assume, without loss of generality, that $\| \boldsymbol{x}_{1} \|_2 \leq 1$ for any input $\boldsymbol{x}_{1}$.

We would like to emphasize the point that these assumptions are on the mapping that we are trying to represent by the network, not the network itself. Thus, assumptions are independent of architecture. Assumptions \ref{assump:smooth} and \ref{assump:inv_lip} mean that the function $\mathcal{F}$ is smooth, Lipschitz continuous, and its inverse is differentiable. The practical relevance of invertibility assumption is justified by the success of reversible networks \cite{Dinh2017DensityNVP,Gomez2017TheActivations,Behrmann2019InvertibleNetworks}. Reversible architectures look for the true mapping $\mathcal{F}$ \emph{only} in the space of invertible functions and it is shown that imposing such strict assumption on the architecture does not hurt its representation ability \cite{Gomez2017TheActivations}. Thus, the mapping $\mathcal{F}$ is either invertible or can be well approximated by an invertible function, in many scenarios. However, unlike the reversible architectures, we do not assume residual blocks or the residual transformations, $F_l(.)$, are invertible, which makes the assumption less strict. Furthermore, our extensive experiments in Section \ref{sec:experiments} show that our theoretical analysis, which is based on these assumptions, hold. This is further empirical justification that these assumptions are relevant in practice. Finally, Assumption \ref{assump:det} is without loss of generality \cite{Bartlett2018RepresentingOptimization, Hardt2017IdentityLearning}. 
\begin{thm}
\label{thm:norm-preservation-nonlinear}
Suppose we want to represent a nonlinear mapping $\mathcal{F}: \mathbb{R}^N \rightarrow \mathbb{R}^N$, satisfying Assumption \ref{assump:nonlinear_map}, with a sequence of $L$ non-linear residual blocks of form $\boldsymbol{x}_{l+1} = \boldsymbol{x}_{l} + F_l(\boldsymbol{x}_{l})$. There exists a solution such that for all residual blocks we have: 
$$
(1 - \delta) \| \frac{\partial \mathcal{E}}{\partial \boldsymbol{x}_{l+1}} \|_2 \leq \| \frac{\partial \mathcal{E}}{\partial \boldsymbol{x}_{l}} \|_2 \leq (1 + \delta) \| \frac{\partial \mathcal{E}}{\partial \boldsymbol{x}_{l+1}} \|_2,
$$
where 
$\delta =  c\frac{\log(2L)}{L}$, $\mathcal{E}(.)$ is the cost function, and $c = c_1\max \{ \alpha \beta ( 1 + \beta), \beta(2 + \alpha) + \alpha  \}$ for some $c_1 > 0$. $\alpha$ and $\beta$ are constants defined in Assumption \ref{assump:nonlinear_map}.
\end{thm}
\begin{proof}
See Section \ref{sec:app_nonlin_proof}.
\end{proof}

This theorem shows that the mapping $\mathcal{F}$ \emph{can} be represented by a sequence of $L$ non-linear residual blocks, such that the norm of the gradient does not change significantly, as it is backpropagated through the layers. \emph{One interesting implication of Theorem \ref{thm:norm-preservation-nonlinear} is that as $L$, the number of layers, increases, $\delta$ becomes smaller and the solution becomes more norm-preserving}. This is a very desirable feature because vanishing or exploding gradient often occurs in deeper network architectures. However, by utilizing residual blocks, as more blocks are stacked, the solution becomes extra norm-preserving.

\begin{figure*}
\centering    
\subfigure[ Block diagram of ResNet]{
\label{subfig:resnet_full}
\includestandalone[mode=buildnew,width=0.95\textwidth]{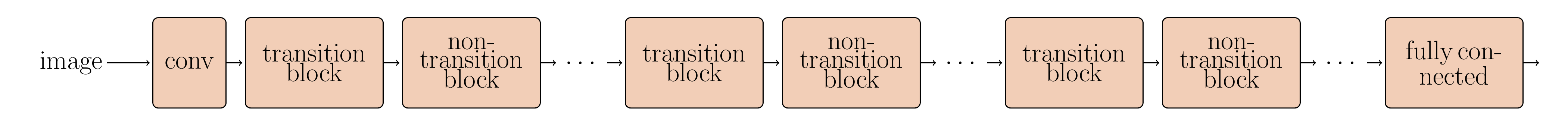}
}     
\subfigure[Residual block with identity mapping (non-transition block)]{
\label{subfig:resnet_orig_block}
\includestandalone[mode=buildnew,width=0.48\textwidth]{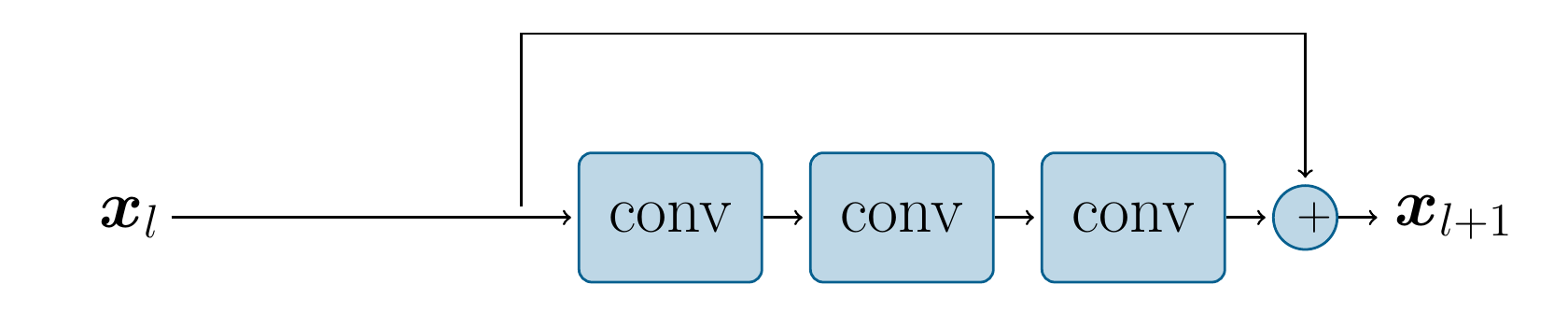}
}  
\subfigure[Origianl ResNet transition block]{
\label{subfig:trans_block}
\includestandalone[mode=buildnew,width=0.48\textwidth]{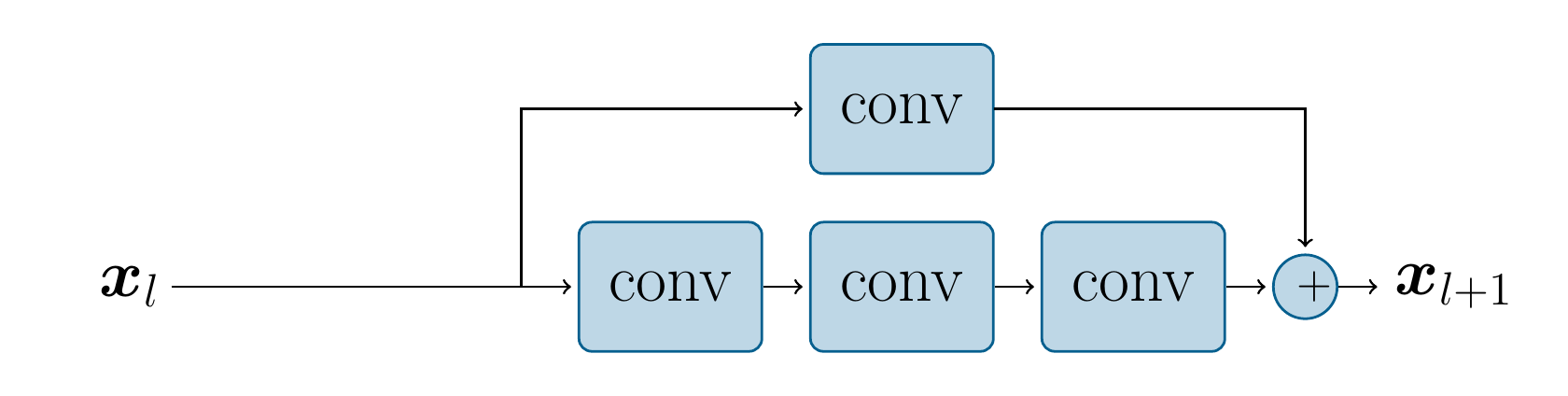}
}  
\subfigure[Proposed transition block]{
\label{subfig:new_trans_block}
\includestandalone[mode=buildnew,width=0.48\textwidth]{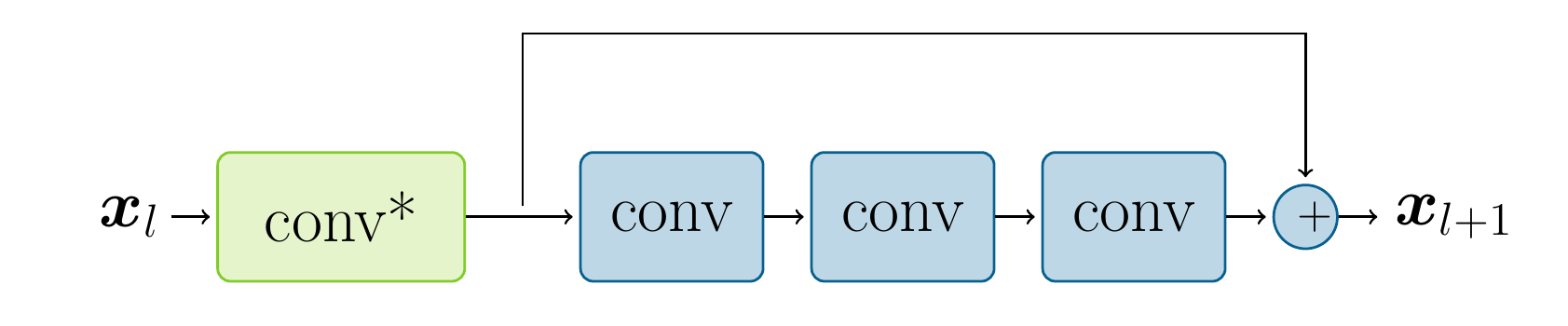}
}
\subfigure[Plain block (transition and non-transition block in a network without skip connections)]{
\label{subfig:plain_block}
\includestandalone[mode=buildnew,width=0.48\textwidth]{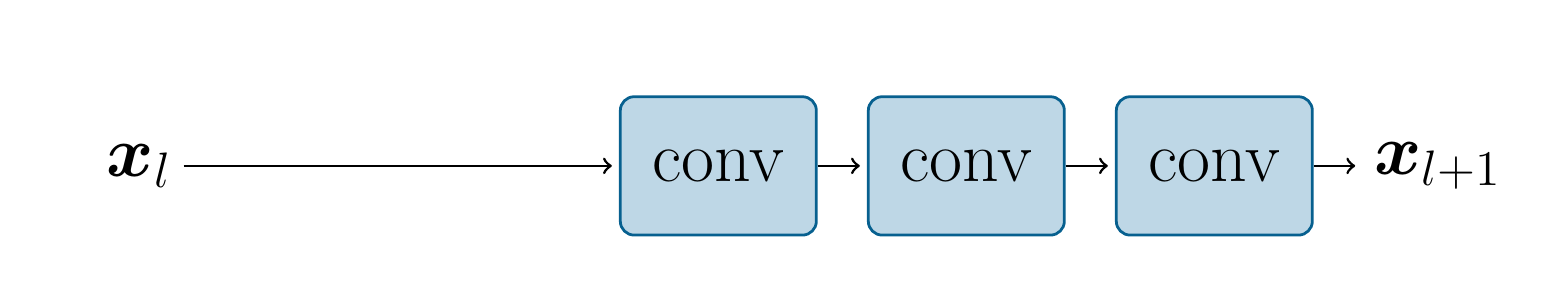}
}
\caption{\small ResNet architecture and its building blocks. Each conv block represents a sequence of batch normalization, ReLU, and convolution layers. conv* block represents the regularized convolution layer.}
\end{figure*}

Now that we proved such a solution exists, we show why residual networks can remain norm preserving throughout the training. For that, we consider the case where $F_l(\boldsymbol{x}_{l})$ consists of two layers of convolution and nonlinearity. The following corollary shows the bound on norm preservation of the residual block depends on the norm of the weights. Therefore, if we bound the optimizer to search only in the space of functions with small norms, we can ensure that the network will remain norm preserving throughout the training. 
Therefore, any critical point in this space is also norm-preserving. 
On the other hand, based on Theorem \ref{thm:norm-preservation-nonlinear}, we know that at least one norm preserving solution exists.
It is also known that, under certain conditions, any critical point achieved during optimization of ResNets is a global minimizer, meaning that it achieves the same loss function value as the global minimum\cite{Hardt2017IdentityLearning,Bartlett2018RepresentingOptimization, Kawaguchi2019DepthResNets}. Thus, this result implies that ResNets are able to maintain norm-preservation throughout the training and if they converge, the solution is a norm-preserving global minimizer.
The conclusions of the corollary can be easily generalized for residual block with more than two layers.
\begin{cor}
\label{cor:norm-preservation-weights}
Suppose a network contains non-linear residual blocks of form $\boldsymbol{x}_{l+1} = \boldsymbol{x}_{l} + \boldsymbol{W}_{l}^{(2)}\boldsymbol{\rho}(\boldsymbol{W}_{l}^{(1)}\boldsymbol{\rho}(\boldsymbol{x}_{l}))$, where $\boldsymbol{\rho}(.)$ is an element-wise non-linear operator with bounded derivative, i.e.,  $0 \leq \frac{\partial \rho_n(\boldsymbol{x}_{l})}{\partial x_{l,n}} \leq c_{\rho}, \forall n = 1,\dots,N$. Then, we have:
$$
(1 - \delta) \| \frac{\partial \mathcal{E}}{\partial \boldsymbol{x}_{l+1}} \|_2 \leq \| \frac{\partial \mathcal{E}}{\partial \boldsymbol{x}_{l}} \|_2 \leq (1 + \delta) \| \frac{\partial \mathcal{E}}{\partial \boldsymbol{x}_{l+1}} \|_2
$$
$$
\text{and} \quad \delta = c_{\rho}^2\| \boldsymbol{W}_{l}^{(1)} \|_2 \| \boldsymbol{W}_{l}^{(2)} \|_2.
$$
\end{cor}
\begin{proof}
See Section \ref{app:proof-nonlinear}
\end{proof}
Here, $\|.\|_2$ is the induced matrix norm and is the largest singular value of the matrix, which is known to be upper bounded by the entry-wise $\ell_2$ norm. This means that norm-preservation is enforced throughout the training process, as long as the norm of the weights are small, not just at the beginning of the training by good initialization. This is the case in practice, since the weights of the network are regularized either explicitly using $\ell_2$ regularization, also known as weight decay, or implicitly by the optimization algorithm \cite{Gunasekar2018ImplicitNetworks,Hoffer2017TrainNetworks}. Thus, the gradients will have very similar magnitudes at different layers, and this leads to well-conditioning and faster convergence \cite{Glorot2010UnderstandingNetworks}.  

Although Theorem \ref{thm:norm-preservation-nonlinear} holds for linear blocks as well, we can derive tighter bounds for linear residual blocks by taking a slightly different approach. For that, we model each linear residual block as:
\begin{equation}
\label{eq:res-block-formulation}
\boldsymbol{x}_{l+1} = \boldsymbol{x}_{l} + \boldsymbol{W}_{l}\boldsymbol{x}_{l},
\end{equation}
where, $\boldsymbol{x}_{l}, \boldsymbol{x}_{l+1} \in \mathbb{R}^N$ are respectively the input and output of the $l^{th}$ residual block,  with dimension $N$. The weight matrix $\boldsymbol{W}_{l} \in \mathbb{R}^{N\times N}$ is the tunable linear transformation. 
The goal of learning is to compute a function $\boldsymbol{y} = \mathcal{M}(\boldsymbol{x},\mathcal{W})$, where $\boldsymbol{x} = \boldsymbol{x}_1$ is the input, $\boldsymbol{y} = \boldsymbol{x}_{L+1}$ is its corresponding output, and $\mathcal{W}$ is the collection of all adjustable linear transformations, i.e., $\boldsymbol{W}_{1},\boldsymbol{W}_{2}, \dots, \boldsymbol{W}_{L}$. In the case of simplified linear residual networks, function $\mathcal{M}(\boldsymbol{x},\mathcal{W})$ is a stack of $L$ residual blocks, as formulated in (\ref{eq:res-block-formulation}). Mathematically speaking, we have:
\begin{equation}
\label{eq:residual-mapping}
\boldsymbol{y} = \mathcal{M}(\boldsymbol{x},\mathcal{W}) = \prod_{l = 1}^{L} (\boldsymbol{I} + \boldsymbol{W}_{l})\boldsymbol{x},
\end{equation}
where $\boldsymbol{I}$ is an $N\times N$ identity matrix. $\mathcal{M}(\boldsymbol{x},\mathcal{W})$ is used to learn a linear mapping $\boldsymbol{R} \in \mathbb{R}^{N \times N}$ from its inputs and outputs. Furthermore, assume that $\boldsymbol{y}$ is contaminated with independent identically distributed (i.i.d) Gaussian noise, i.e.,  $\hat{\boldsymbol{y}} = \boldsymbol{R}\boldsymbol{x} + \boldsymbol{\epsilon}$, where $\boldsymbol{\epsilon}$ is a zero mean noise vector with covariance matrix $\boldsymbol{I}$. Hence, our objective is to minimize the expected error of the maximum likelihood estimator as:
\begin{equation}
\label{eq:optimization}
\min_{\mathcal{W}} \mathcal{E}(\mathcal{W}) = \mathbb{E}\{ \frac{1}{2} \| \hat{\boldsymbol{y}} - \mathcal{M}(\boldsymbol{x},\mathcal{W}) \|_2^2 \},
\end{equation}
where the expectation $\mathbb{E}$ is with respect to the population $(\boldsymbol{x},\boldsymbol{y})$. The following theorem states the  bound on the norm preservation of the linear residual blocks. \par

\begin{thm}
\label{thm:linear-norm-preservation}
For learning a linear map, $\boldsymbol{R} \in \mathbb{R}^{N \times N}$, between its input $\boldsymbol{x}$ and output $\boldsymbol{y}$ contaminated with i.i.d Gaussian noise, using a network consisting of $L$ linear residual blocks of form $\boldsymbol{x}_{l+1} = \boldsymbol{x}_{l} + \boldsymbol{W}_{l}\boldsymbol{x}_{l}$, there exists a global optimum for $\mathcal{E}(.)$, as defined in (\ref{eq:optimization}), such that for all residual blocks we have
$$
(1 - \delta) \| \frac{\partial \mathcal{E}}{\partial \boldsymbol{x}_{l+1}} \|_2 \leq \| \frac{\partial \mathcal{E}}{\partial \boldsymbol{x}_{l}} \|_2 \leq (1 + \delta) \| \frac{\partial \mathcal{E}}{\partial \boldsymbol{x}_{l+1}} \|_2
$$
for $L \geq 3\gamma$, where $\delta = \frac{c}{L}$, $c = 2(\sqrt[]{\pi} + \sqrt[]{3\gamma})^2$, and $\gamma = \max (|\log \sigma_{max}(\boldsymbol{R})|,|\log \sigma_{min}(\boldsymbol{R})|)$, where  $\sigma_{max}(\boldsymbol{R})$  and $\sigma_{min}(\boldsymbol{R})$, respectively, are maximum and minimum singular values of $\boldsymbol{R}$.
\end{thm}
\begin{proof}
See Section \ref{sec:app_lin_proof}
\end{proof}

Similar to the nonlinear residual blocks, the linear blocks become more norm-preserving as we increase the depth. However, the linear blocks become norm-preserving at a faster rate. The gradient norm ratio for the linear blocks approaches $1$ with a rate of $\mathcal{O}(\frac{1}{L})$, while this ratio for nonlinear blocks approaches $1$ with a rate of $\mathcal{O}(\frac{\log(L)}{L})$.

\section{Procrustes Residual Network  }
\label{sec:Transition Block}

As depicted in Figure \ref{subfig:resnet_full}, residual networks contain four different types of blocks: 
\begin{enumerate*}[label=(\roman*)]
\item convolution layer (first layer),
\item fully connected layer (last layer),
\item transition blocks (which change the dimension) as depicted in Figure \ref{subfig:trans_block}, and
\item residual blocks with identity skip connection, as illustrated in Figure \ref{subfig:resnet_orig_block}, which we also refer to as non-transition blocks.
\end{enumerate*}
Theoretical investigation presented in Section \ref{sec:study} holds only for residual blocks with identity mapping as the skip connection. Such identity skip connection cannot be used in the transition blocks, since the size of the input is not the same as the size of output. If the  benefits of residual networks can be explained, at least partly, by norm-preservation, then one can improve them by alternative methods for preserving the norm.  In this section, we propose to modify the transition blocks of ResNet architecture, to make them norm-preserving. Due to multiplicative effect through the layers, making these layers norm-preserving may be important, although they make up a small portion of the network. In the following, we discuss how to preserve the norm of the back-propagated gradients across all the blocks of the network.

As depicted in Figure \ref{subfig:trans_block}, in the original ResNet architecture, the dimension changing blocks, also known as transition blocks, use $1\times1$ convolution with stride of $2$ in their skip connections to match the dimension of input and output activations. Such transition blocks are not norm-preserving in general. \par 

Figure \ref{subfig:new_trans_block} shows the block diagram of the proposed norm-preserving transition block. To change the dimension in a norm-preserving manner, we utilize a norm preserving convolution layer, conv*. For that, we project the convolution kernel onto the set of norm preserving kernels by setting its singular values. Here, we show how we can make the convolution layer norm preserving by regularizing the singular values, without using singular value decomposition. Specifically, the gradient of a convolution layer with kernel of  size  $k$, with $c$ input channels, and $d$ output channels can be formulated as:
\begin{equation}
\label{eq:gradient-conv}
\Delta_{\boldsymbol{x}} = \hat{\boldsymbol{W}} \Delta_{\boldsymbol{y}},
\end{equation}
where $\Delta_{\boldsymbol{x}}$ and $\Delta_{\boldsymbol{y}}$ respectively are the gradients with respect to the input and output of the convolution. $\Delta_{\boldsymbol{y}}$ is an $n^2d$ dimensional vector, representing $n \times n$ pixels in $d$ output channels, and $\Delta_{\boldsymbol{x}}$ is an $n^2c$ dimensional vector, representing the gradient at the input. Furthermore, $\hat{\boldsymbol{W}}$ is an $n^2c \times n^2d $ dimensional matrix embedding the back-propagation operation for the convolution layer. We can represent this linear transformation as:
\begin{equation}
\label{eq:gradient-conv-svd}
\Delta_{\boldsymbol{x}} = \sum_{j = 1}^{n^2c} \sigma_j\boldsymbol{u}_j<\Delta_{\boldsymbol{y}}, \boldsymbol{v}_j>,
\end{equation}
where $\{ \sigma_j, \boldsymbol{u}_j, \boldsymbol{v}_j\}$ is the set of singular values and singular vectors of $\hat{\boldsymbol{W}}$. 
Furthermore, since the set of the right singular vectors, i.e., $\{ \boldsymbol{v}_j \}$, is an orthonormal basis set for $\Delta_{\boldsymbol{y}}$, we can write the gradient at the output as:
$$
\Delta_{\boldsymbol{y}} = \sum_{j = 1}^{n^2d} \boldsymbol{v}_j<\Delta_{\boldsymbol{y}}, \boldsymbol{v}_j>.
$$

Thus, we can compute the expected value of the norm of the gradients as:
$$
\begin{aligned}
& \mathbb{E} [\| \Delta_{\boldsymbol{x}} \|_2^2]
&& = \sum_{j = 1}^{n^2c} \sigma_j^2 \mathbb{E} [ |<\Delta_{\boldsymbol{y}} , \boldsymbol{v}_j>|^2], \\
& \mathbb{E} [\| \Delta_{\boldsymbol{y}} \|_2^2]
&& = \sum_{j = 1}^{n^2d} \mathbb{E} [ |<\Delta_{\boldsymbol{y}} , \boldsymbol{v}_j>|^2],
\end{aligned}
$$
where we use the fact that $\boldsymbol{u}_i^T\boldsymbol{u}_j = \boldsymbol{v}_i^T\boldsymbol{v}_j = 0$ for $i \neq j$ and $\boldsymbol{u}_j^T\boldsymbol{u}_j = \boldsymbol{v}_j^T\boldsymbol{v}_j = 1$ and the expectation is over the data population. 
We propose to preserve the norm of the gradient, i.e., $\mathbb{E} [\| \Delta_{\boldsymbol{x}} \|_2^2] = \mathbb{E} [\| \Delta_{\boldsymbol{y}} \|_2^2]$,  by setting all the non-zero singular values to $\sigma$. It is easy to show that we can achieve this by setting
\begin{equation}
\label{eq:norm-preserve-cond}
\sigma^2 = \frac{\sum_{j = 1}^{n^2d} \mathbb{E} [ |<\Delta_{\boldsymbol{y}} , \boldsymbol{v}_j>|^2]}
{\sum_{j, \sigma_j \neq 0}^{} \mathbb{E} [ |<\Delta_{\boldsymbol{y}} , \boldsymbol{v}_j>|^2]},
\end{equation}
where the summation in the denominator is over the singular vectors $\boldsymbol{v}_j$ corresponding to the nonzero singular values, i.e., $\sigma_j \neq 0$. The ratio in \eqref{eq:norm-preserve-cond} is the ratio of expected energy of $\Delta_{\boldsymbol{y}}$, i.e. $\mathbb{E} [\| \Delta_{\boldsymbol{y}} \|_2^2]$, divided by the portion of energy that does not lie in the null space of $\hat{\boldsymbol{W}}$.
We make the assumption that this ratio can be approximated by $\frac{n^2d}{n^2\min(d,c)}$. 
This assumption implies that about $\frac{n^2\min(d,c)}{n^2d}$ of the total energy of $\Delta_{\boldsymbol{y}}$ will lie in the $n^2\min(d,c)$-dimensional subspace, corresponding to orthonormal basis set $\{ \boldsymbol{v}_j  | \sigma_j \neq 0\}$, of our $n^2d$-dimensional space. It is easy to notice that the assumption holds if the energy of $\Delta_{\boldsymbol{y}}$ is distributed uniformly among the directions in the basis set $\{ \boldsymbol{v}_j \}$. But, since we are taking the sum over a large number of bases, it can also hold with high probability in cases where there is some variation in the distribution of energies along different directions.
This is not a strict assumption in high dimensional spaces and we will investigate the practical relevance of this assumption in a real-world setting shortly.
Thus, we can achieve norm preservation by setting all the nonzero singular values to $\sqrt{\frac{d}{\min(d,c)}}$. We can enforce this equality without using singular value decomposition.  For that, we use the following theorem from \cite{Sedghi2019TheLayers}. This theorem states that the singular values of the convolution operator can be calculated by finding the singular values of the Fourier transform of the slices of the convolution kernels.

\begin{thm}
\label{thm:conv-singular-values}
(Theorem 6 from \cite{Sedghi2019TheLayers})
For any convolution kernel $\boldsymbol{K} \in \mathbb{R}^{k \times k \times d \times c}$ acting on an $n \times n \times d$ input, let $\hat{\boldsymbol{W}}$ be the matrix encoding the linear transformation
computed by a convolutional layer parameterized by $\boldsymbol{K}$. Also, for each $u,v \in [n]\times[n]$, let $\boldsymbol{P}^{(u,v)} \in \mathbb{C}^{d \times c}$ be the matrix given by $\boldsymbol{P}^{(u,v)}_{i,j} = (\mathcal{F}_n(\boldsymbol{K}_{:,:,i,j}))_{u,v}$, where $\mathcal{F}_n(.)$ is the operator describing an $n \times n$ 2D Fourier transform. Then, the set of singular values of $\hat{\boldsymbol{W}}$ is the union (allowing repetitions) of all the singular values of $\boldsymbol{P}^{(u,v)}$ matrices $\forall u,v$.
\end{thm}

\begin{proof}
See \cite{Sedghi2019TheLayers}.
\end{proof}

Hence, to satisfy the condition (\ref{eq:norm-preserve-cond}), we can set all the nonzero singular values of $\boldsymbol{P}^{(u,v)}$ to $\sqrt{\frac{d}{\min(d,c)}}$ for all $u$ and $v$. This can be done by finding the matrix $\hat{\boldsymbol{P}}^{(u,v)}$ that minimizes $\| \boldsymbol{P}^{(u,v)} - \hat{\boldsymbol{P}}^{(u,v)} \|_F^2$, such that $\hat{\boldsymbol{P}}^{{(u,v)}^T}\hat{\boldsymbol{P}}^{(u,v)} = \frac{d}{\min(d,c)}\boldsymbol{I}$, where $\|.\|_F$ denotes the Frobenius norm and $\boldsymbol{I}$ is a $c \times c$ identity matrix. It can be shown that the solution to this problem is given by
\begin{equation}
\label{eq:procrustes-solution}
\hat{\boldsymbol{P}}^{(u,v)} = \sqrt{\frac{d}{\min(d,c)}} \boldsymbol{P}^{(u,v)}(\boldsymbol{P}^{{(u,v)}^T}\boldsymbol{P}^{(u,v)})^{-\frac{1}{2}}.
\end{equation}

This is closely related to \emph{Procrustes} problems, in which the goal is to find the closest orthogonal matrix to a given matrix \cite{Gower2004ProcrustesProblems}. Finding the inverse of the square root of product $\boldsymbol{P}^{{(u,v)}^T}\boldsymbol{P}^{(u,v)}$ can be computationally expensive, specifically for large number of channels $c$. Thus, we exploit an iterative algorithm that computes the inverse of the square root using only matrix multiplications. Specifically, one can use the following iterations to compute $(\boldsymbol{P}^{{(u,v)}^T}\boldsymbol{P}^{(u,v)})^{-\frac{1}{2}}$ \cite{Higham1997StableRoot}:
\begin{equation}
\label{eq:inv-sqrt-iters}
\begin{aligned}
&\boldsymbol{T}_{k} = 3 \boldsymbol{I} - \boldsymbol{Z}_{k}\boldsymbol{Y}_{k},\\
&\boldsymbol{Y}_{k+1} = \frac{1}{2} \boldsymbol{Y}_{k}\boldsymbol{T}_{k}, \\
&\boldsymbol{Z}_{k+1} = \frac{1}{2} \boldsymbol{T}_{k}\boldsymbol{Z}_{k},
\end{aligned}
\end{equation}
for $k = 0,1,\dots$ and the iterators are initialized as:
$$
\boldsymbol{Y}_0 = \boldsymbol{P}^{{(u,v)}^T}\boldsymbol{P}^{(u,v)},
\boldsymbol{Z}_0 = \boldsymbol{I}.
$$
It has been shown that $\boldsymbol{Z}_{k}$ converges to $(\boldsymbol{P}^{{(u,v)}^T}\boldsymbol{P}^{(u,v)})^{-\frac{1}{2}}$ quadratically \cite{Higham1997StableRoot}. Since the iterations only involve matrix multiplication, they can be implemented efficiently on GPUs.

\begin{algorithm}[t]
\renewcommand{\algorithmicrequire}{\textbf{Input:}}
\renewcommand{\algorithmicensure}{\textbf{Output:}}
\begin{algorithmic}[1]
\REQUIRE Convolution kernel $\boldsymbol{K}$ at the current iteration
\STATE Perform the gradient descent step on the kernel $\boldsymbol{K}$. \\
\STATE Calculate $\boldsymbol{P}^{(u,v)}$ for each $u,v \in [n]\times[n]$ as $\boldsymbol{P}^{(u,v)}_{i,j} = (\mathcal{F}_n(\boldsymbol{K}_{:,:,i,j}))_{u,v}$. \\
\STATE Compute $(\boldsymbol{P}^{{(u,v)}^T}\boldsymbol{P}^{(u,v)})^{-\frac{1}{2}}$ using \eqref{eq:inv-sqrt-iters}. \\
\STATE Calculate $\hat{\boldsymbol{P}}^{(u,v)}$ using \eqref{eq:procrustes-solution}.\\
\STATE Update $\boldsymbol{K}$ using the inverse 2D Fourier transform of $\hat{\boldsymbol{P}}^{(u,v)}$. \\
\end{algorithmic}
\caption{\small Update rules for transition kernels at each iteration}
\label{alg:iteration}
\end{algorithm}

Thus, to keep the convolution kernels norm preserving throughout the training, at each iteration, we compute the matrices $\boldsymbol{P}^{(u,v)}$ and set the nonzero singular values using (\ref{eq:procrustes-solution}). Algorithm \ref{alg:iteration} summarizes the operations performed at each iteration on the kernels of the regularized convolution layers. To keep the desired norm-preservation property after performing the gradient descent step, such as SGD, Adam, etc, the proposed scheme is used to re-enforce norm-preservation on the updated kernel. In this manner, we can maintain norm-preservation, while updating the kernel during the training.
Our experiments in Section \ref{sec:experiments} show that performing the proposed projection on the transition block of deep ResNets increases the training time by less than 8\%. Also, since the number of transition blocks are independent of depth, the deeper the network gets, the computational overhead of the proposed modification becomes less significant.
Figure \ref{subfig:new_trans_block} shows the diagram of the proposed transition block, where a regularized convolution layer, conv*, is used to change the dimension. Hence, we are able to exploit a regular residual block with identity mapping, which is norm preserving. \par 

Similar to \cite{He2015DelvingClassification}, to take into the account the effect of a ReLU nonlinearity and to make a Conv-Relu layer norm-presering, we just need to add a factor of $\sqrt{2}$ to the singular values and set them to $\sqrt{\frac{2d}{\min(d,c)}}$. Intuitively, the element-wise ReLU sets half of the units to zero on average, making the expected value of the energy of the gradient equal to 
$\mathbb{E} [\| \Delta_{\boldsymbol{x}} \|_2^2] 
= \frac{1}{2}\sum_{j = 1}^{n^2c} \sigma_j^2 \mathbb{E} [ |<\Delta_{\boldsymbol{y}} , \boldsymbol{v}_j>|^2]$. Therefore, to compensate this, we need to satisfy this condition:
$$
\frac{1}{2} \sum_{j = 1}^{n^2c} \sigma_j^2 \mathbb{E} [ |<\Delta_{\boldsymbol{y}} , \boldsymbol{v}_j>|^2] = \sum_{j = 1}^{n^2d} \mathbb{E} [ |<\Delta_{\boldsymbol{y}} , \boldsymbol{v}_j>|^2]
$$

It is also worthwhile to mention that since we are trying to preserve the norm of the backward signal, the variable $n$ in Theorem \ref{thm:conv-singular-values} represents the size of feature map size at the output of the convolution. 

To evaluate the effectiveness of the proposed projection, we design the following experiment. We perform the projection on the convolution layers of a small $3$-layer network. The network consists of $3$ convolutional layers, followed by ReLU non-linearity. To examine the gradient norm ratio for different number of input and output channels, the second layer is a $3 \times 3$ convolution with $c$ input channels and $d$ output channels. The first and third layers are $1 \times 1$ convolutions to change the number of channels and to match the size of the input and output layers. 
Figure \ref{fig:ratio_channels} shows the gradient norm ratio, i.e., $\small \| \frac{\partial \mathcal{E}}{\partial \boldsymbol{x}_{l+1}} \|_2$ to $\| \frac{\partial \mathcal{E}}{\partial \boldsymbol{x}_{l}} \|_2$, for different values of $c$ and $d$ at $10^{\text{th}}$ training epoch on CIFAR-10, with and without the proposed projection. The values are averaged over $10$ different runs. 

It is evident that the proposed projection enhances the norm preservation of the Conv-ReLU layer, as it moves the gradient norm ratios toward $1$. The only failure case is for networks with very small $c$ and $c \ll d$. This is because, due to the smaller size of the space, our assumption that the energy of the signal in the $n^2c$ dimensional subspace, corresponding to the $n^2c$ non-zero singular values, is approximately $\frac{n^2c}{n^2d}$ of the total energy of the signal, is violated with higher probability. However, in more practical settings, where the number of channels is large and the assumption is held, the proposed projection performs as expected. This experiment illustrates the validity of our analysis as well as the effectiveness of the proposed projection for such practical scenarios. In the next section, we demonstrate the advantages of the proposed method for image classification task. \par 

\begin{figure}
\centering
\subfigure[ \small With the Proposed Regularization ]{
\label{subfig:ratio_channels_with_regul}
\includegraphics[width=\figurewidth]{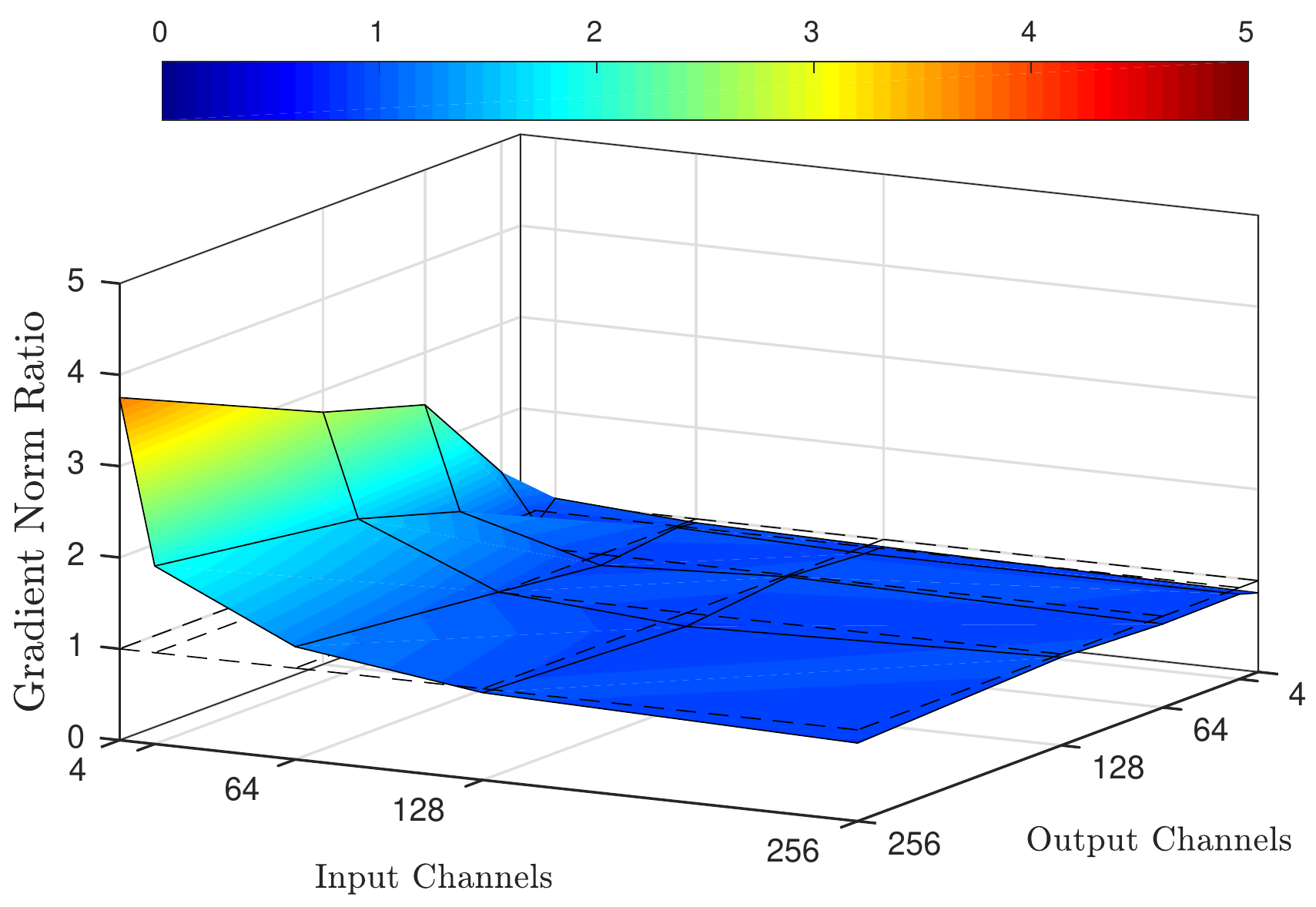}
}
\subfigure[ \small Without Regularization ]{
\label{subfig:ratio_channels_without_regul}
\includegraphics[width=\figurewidth]{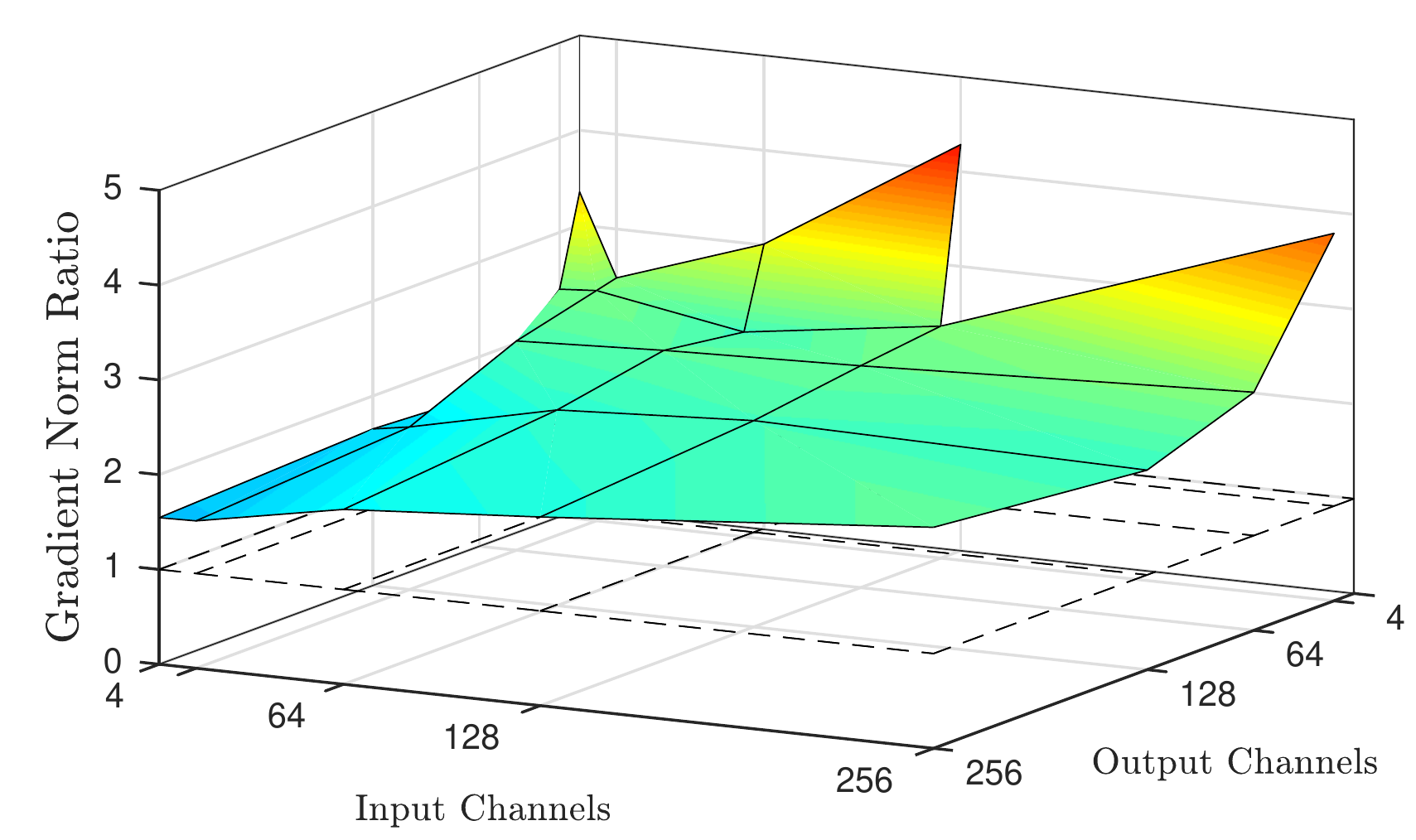}
}
\caption{\small The ratio of gradient norm at output to gradient norm at input, i.e., $\small \| \frac{\partial \mathcal{E}}{\partial \boldsymbol{x}_{l+1}} \|_2$ to $\| \frac{\partial \mathcal{E}}{\partial \boldsymbol{x}_{l}} \|_2$,  of a convolution layer for different number of input and output channels at $10^{\text{th}}$ training epoch \subref{subfig:ratio_channels_with_regul} with, and \subref{subfig:ratio_channels_without_regul} without the proposed regularization on the singular values of the convolution. }
\label{fig:ratio_channels}
\end{figure}
\section{Experiments}
\label{sec:experiments}
To validate our theoretical investigation, presented in Section \ref{sec:study}, and to empirically demonstrate the behavior and effectiveness of the proposed modifications, we experimented with Residual Network (ResNet) and the proposed Procrustes Residual Network (ProcResNet) architectures on CIFAR10 and CIFAR100 datasets. Training and testing datasets contain 50,000 and 10,000 images of visual classes, respectively \cite{Krizhevsky2009LearningImages.2009}. Standard data augmentation (flipping and shifting), same as \cite{He2016IdentityNetworks,He2016DeepRecognition,Huang2017DenselyNetworks}, is adopted. Furthermore, channel means and standard deviations are used to normalize the images. The network is trained using stochastic gradient descent. The weights are initialized using the method proposed in \cite{He2015DelvingClassification} and the initial learning rate is $0.1$. Batch size of $128$ is used for all the networks. The weight decay is $10^{-4}$ and momentum is $0.9$. The results are based on the top-1 classification accuracy. \par 

Experiments are performed on three different network architectures: 
\begin{enumerate*}
\item \textbf{ResNet} contains one convolution layer, $L$ residual blocks, three of which are transition blocks, and one fully connected layer. Each residual block consists of three convolution layers, as depicted in Figure \ref{subfig:resnet_orig_block} and Figure \ref{subfig:trans_block}, resulting in a network of depth $3L + 2$. This is the same architecture as in \cite{He2016IdentityNetworks}.
\item \textbf{ProcResNet} has the same architecture as ResNet, except the transition layers are modified, as explained in Section \ref{sec:Transition Block}. In this design, $3$ extra convolution layers are added to the network. However, we can use the first convolution layer of the original ResNet design to match the dimensions and only add two extra layers. This leads to a network of depth $3L+4$. 
\item \textbf{Plain} network is also same as ResNet without the skip connection in all the $L$ residual blocks, as shown in Figure \ref{subfig:plain_block}.
\end{enumerate*}

Furthermore, to decrease the computational burden of the proposed regularization, we perform the projection, as described in Section \ref{sec:Transition Block}, every $2$ iterations. This reduces the computation time significantly without hurting the performance much. In this setting, performing the proposed regularization increases the training time for ResNet164 about $7.6\%$. However, since we perform the regularization only on three blocks, regardless of the depth, as the network becomes deeper the computational overhead becomes less significant. For example, implementing the same projections on ResNet1001 increases the training time by only $3.5\%$. This is significantly less computation compared to regularization using SVD, which leads to $53\%$ and $23\%$ training time overhead for ResNet164 and ResNet1001, respectively\footnote{An implementation of ProcResNet is provided here: \url{https://github.com/zaeemzadeh/ProcResNet}}.

\subsection{Norm-Preservation}
\label{subsec:exp_ratio}
\begin{figure*}
\centering     
\subfigure[ Plain20 ]{
\label{subfig:ratio-epoch-plain20}
\includegraphics[width=\tilefigwidth]{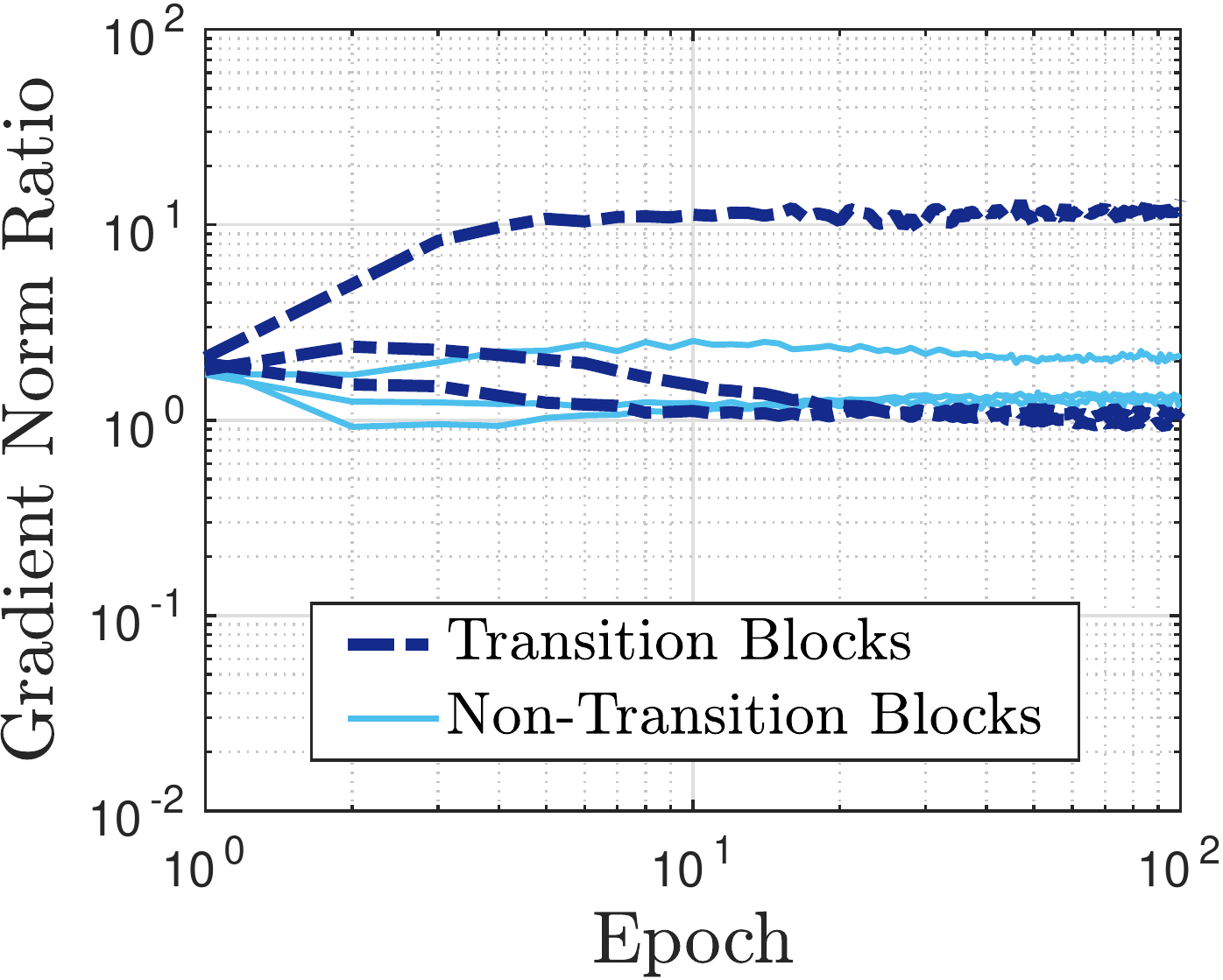}
}
\subfigure[ ResNet20]{
\label{subfig:ratio-epoch-resnet20}
\includegraphics[width=\tilefigwidth]{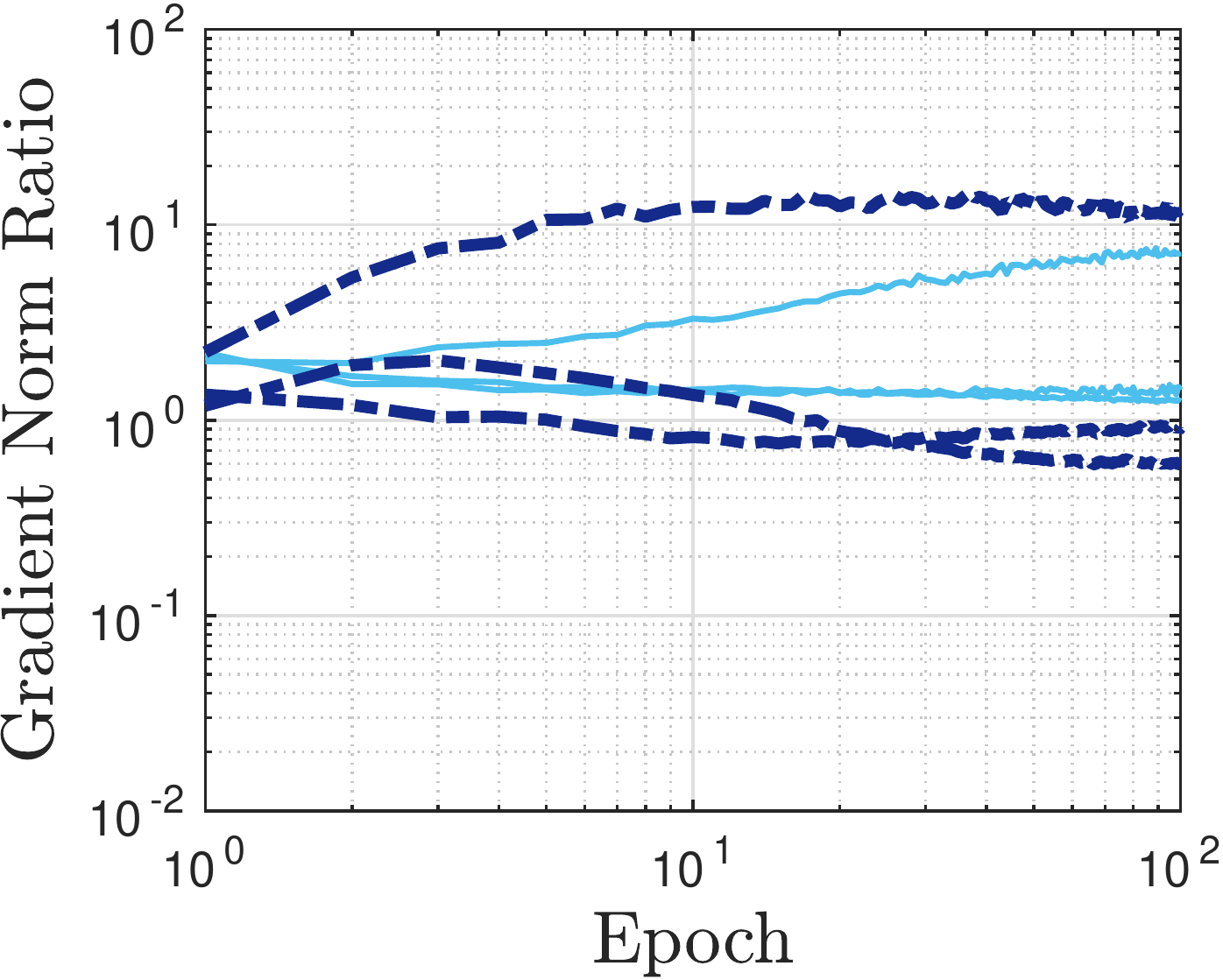}
}
\subfigure[ ProcResNet22]{
\label{subfig:ratio-epoch-flatnet20}
\includegraphics[width=\tilefigwidth]{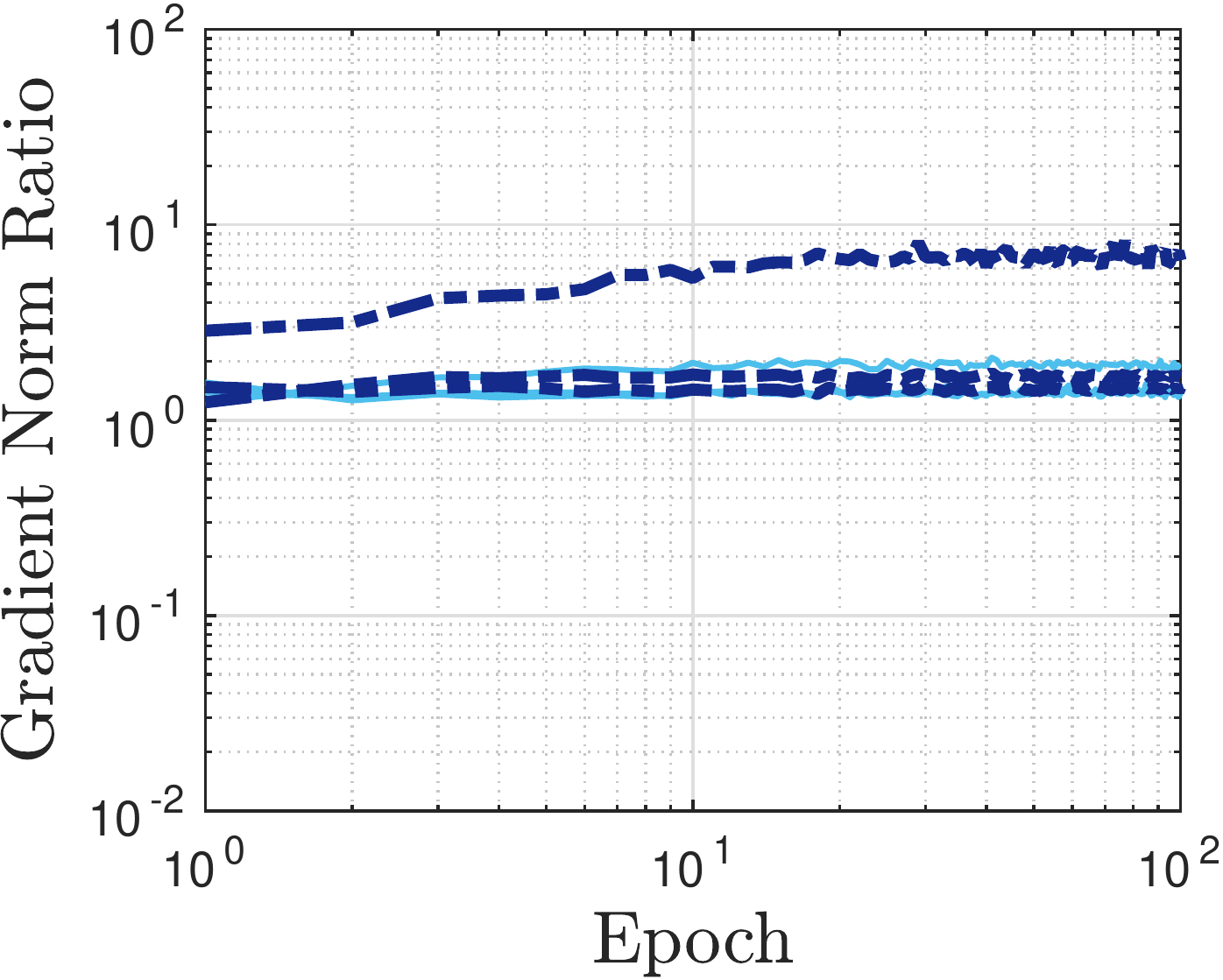}
}    
\subfigure[ Plain83]{
\label{subfig:ratio-epoch-plain83}
\includegraphics[width=\tilefigwidth]{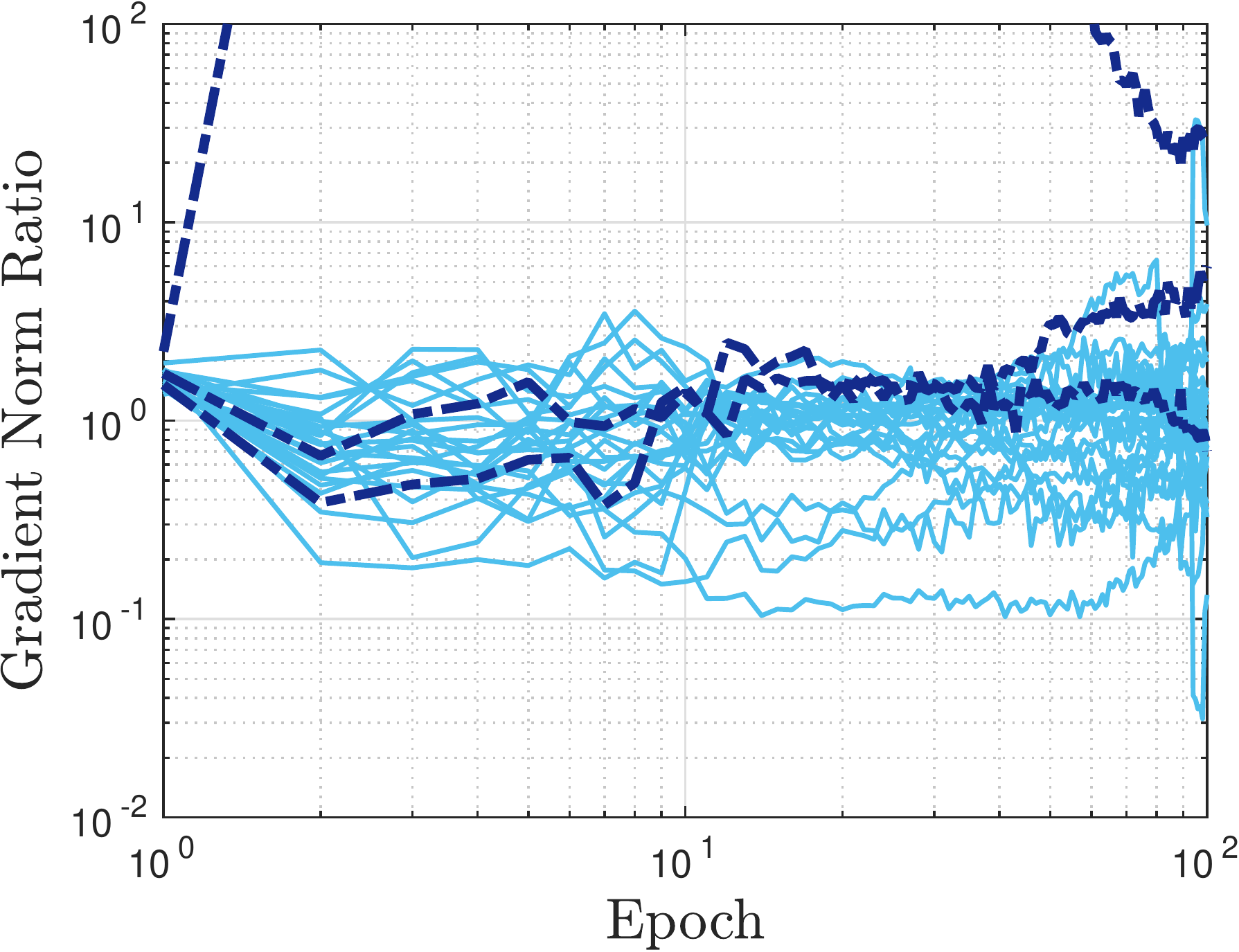}
}
\subfigure[ ResNet83]{
\label{subfig:ratio-epoch-resnet83}
\includegraphics[width=\tilefigwidth]{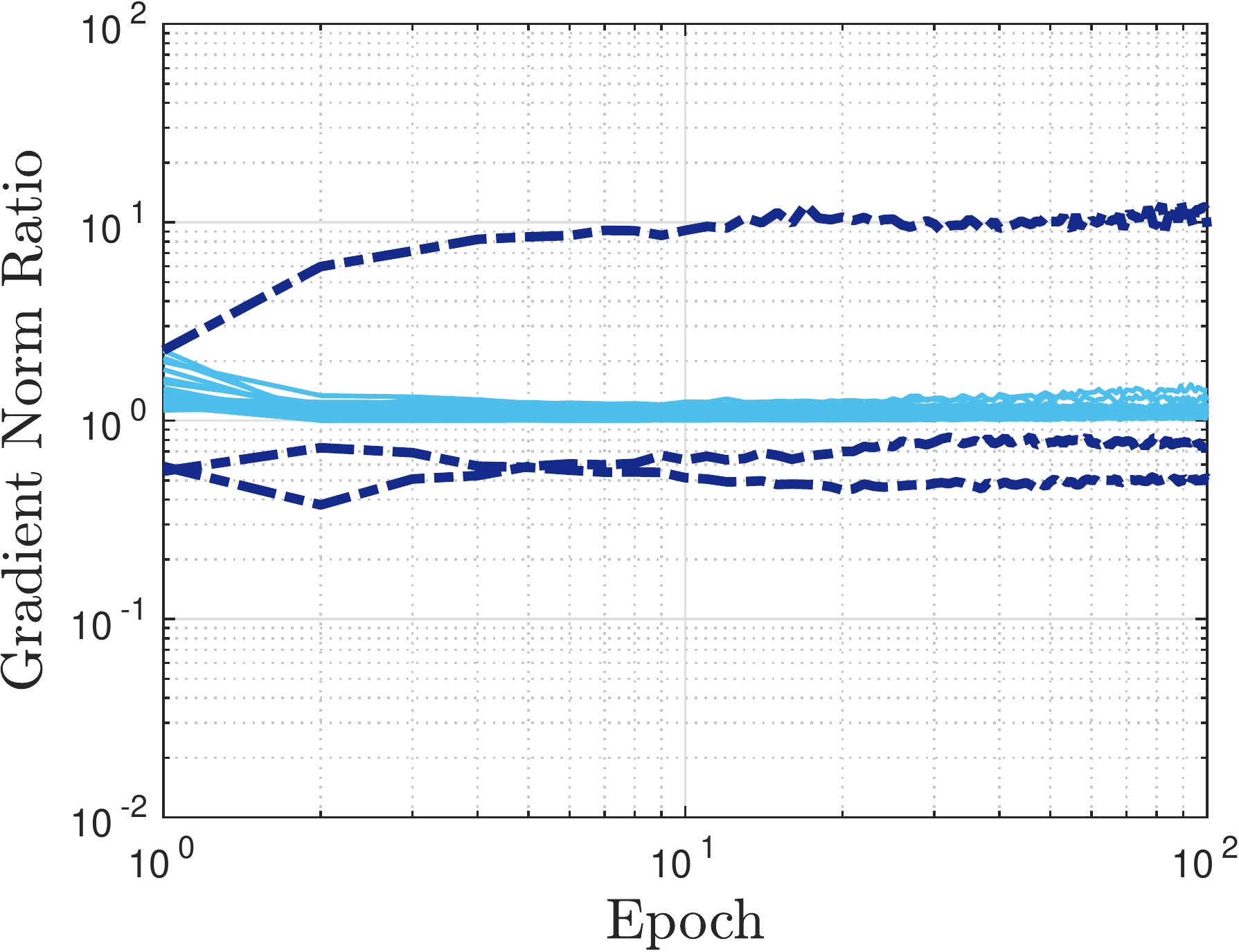}
}
\subfigure[ ProcResNet85]{
\label{subfig:ratio-epoch-flatnet83}
\includegraphics[width=\tilefigwidth]{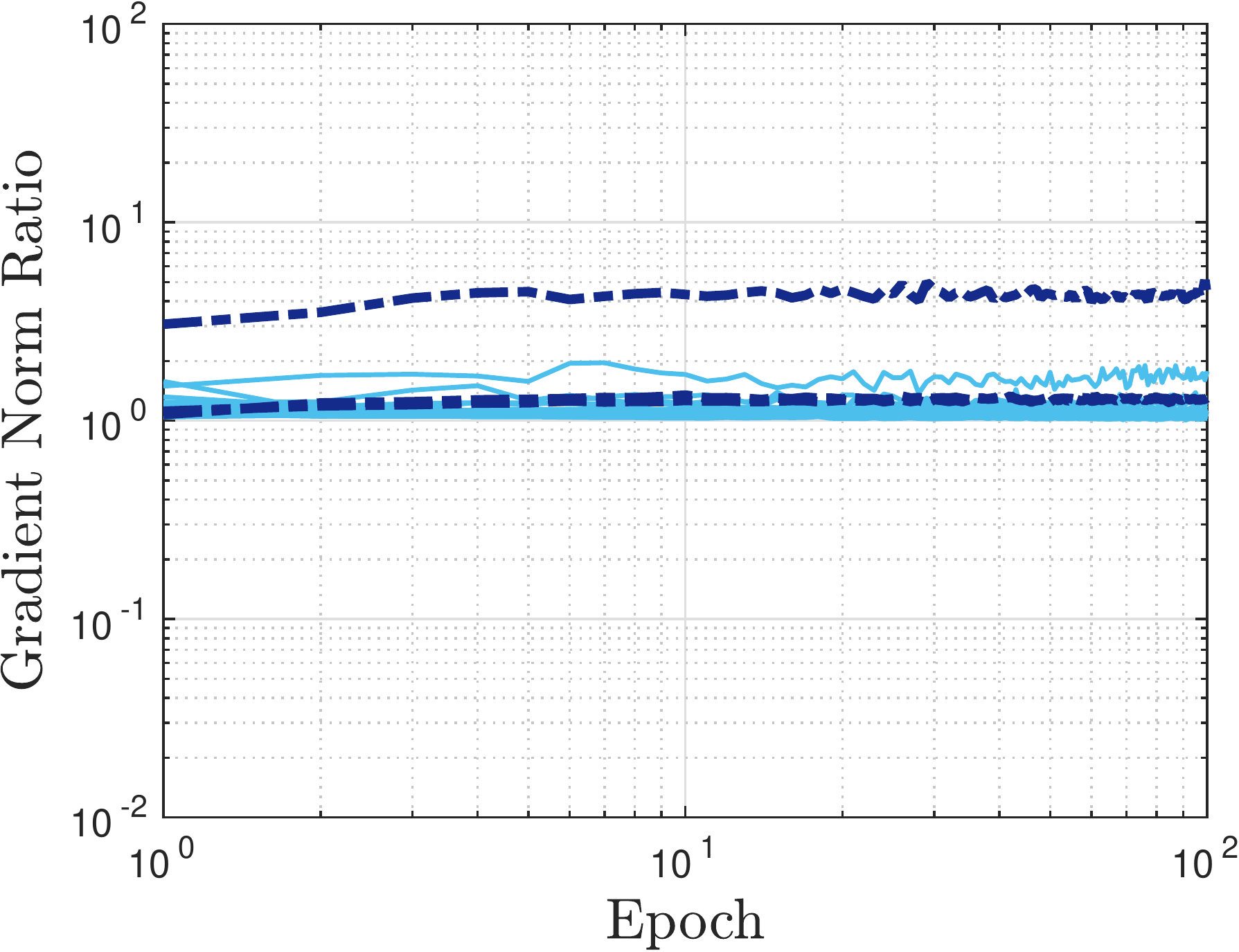}
}    
\subfigure[ Plain164]{
\label{subfig:ratio-epoch-plain164}
\includegraphics[width=\tilefigwidth]{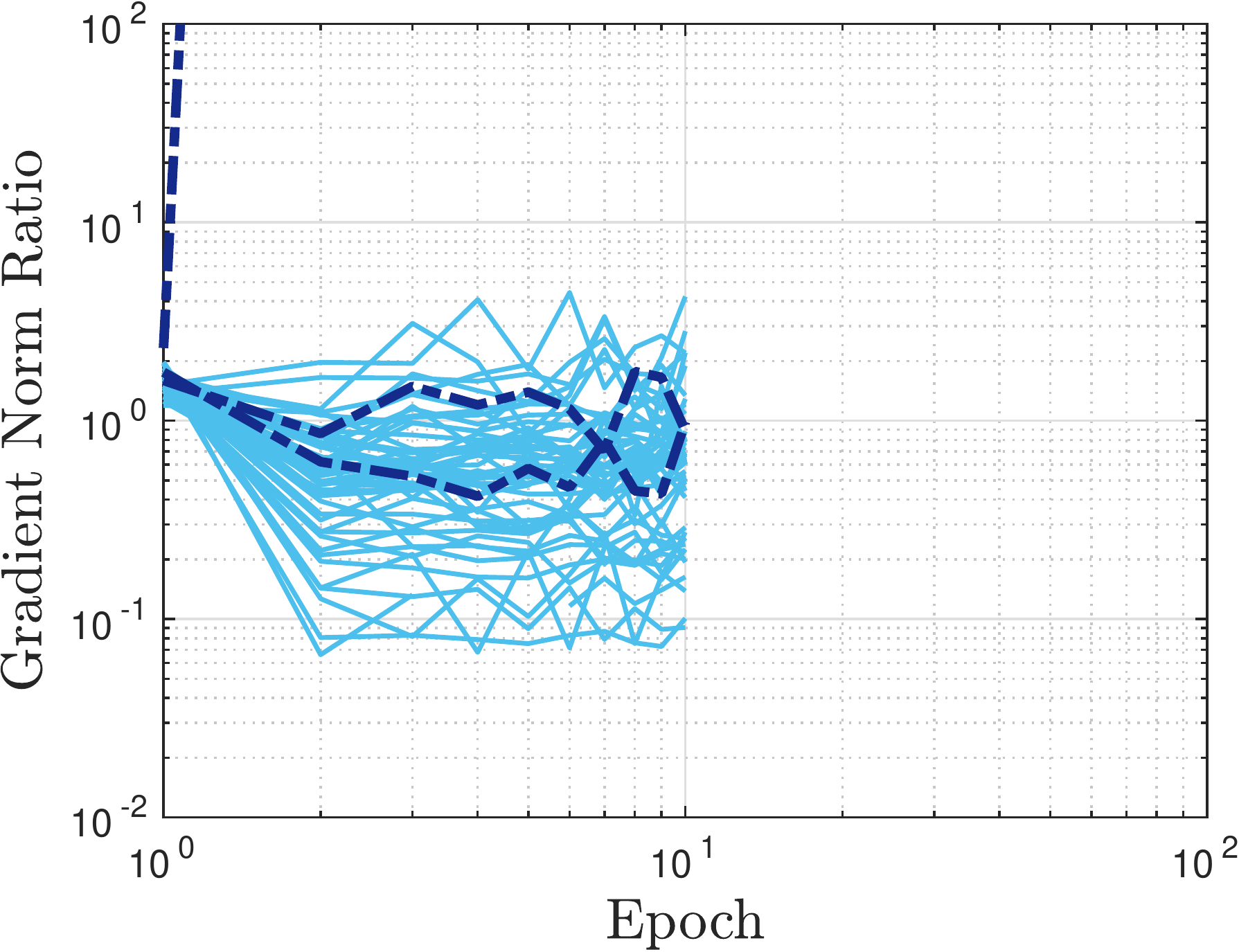}
}
\subfigure[ ResNet164]{
\label{subfig:ratio-epoch-resnet164}
\includegraphics[width=\tilefigwidth]{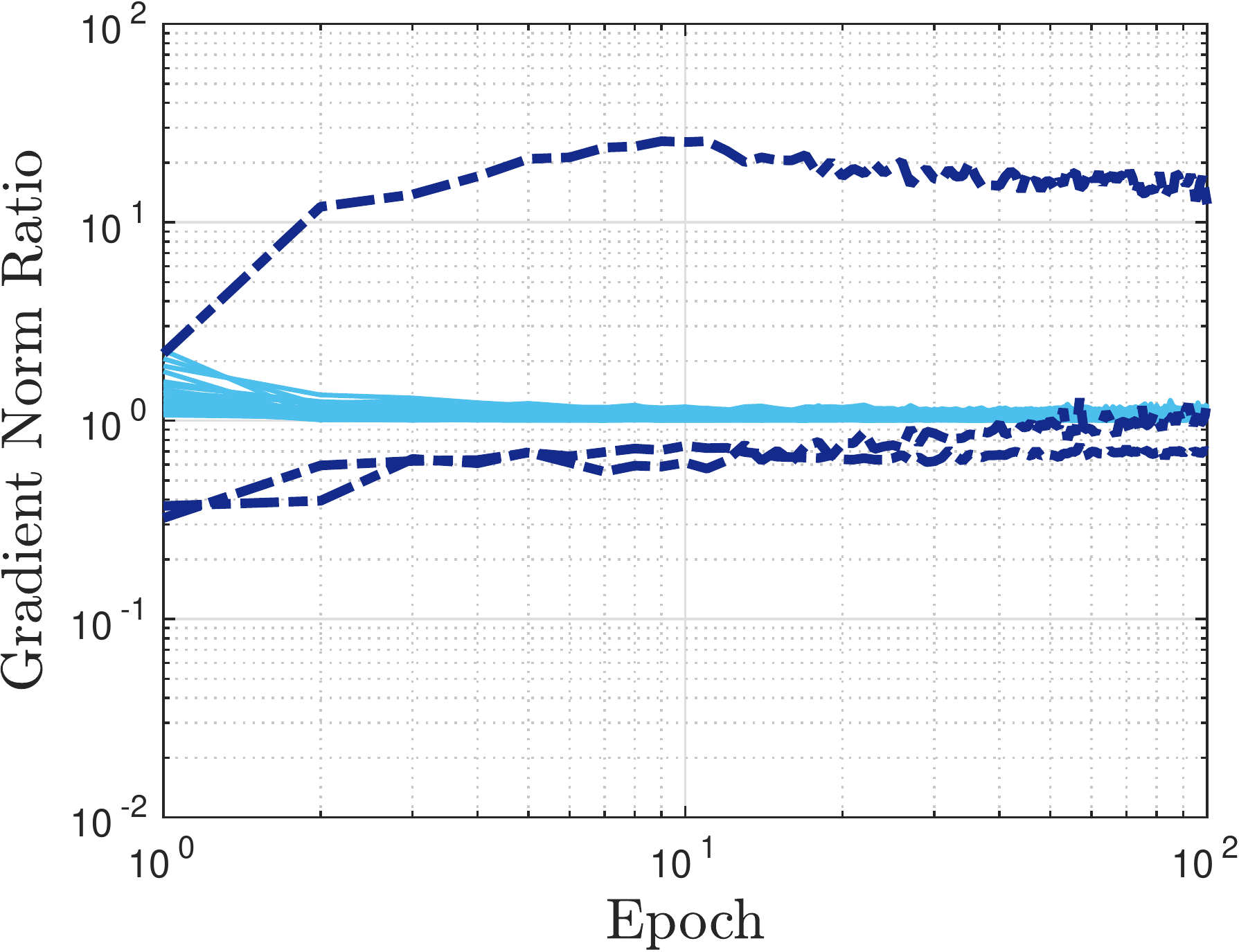}
}
\subfigure[ ProcResNet166]{
\label{subfig:ratio-epoch-flatnet164}
\includegraphics[width=\tilefigwidth]{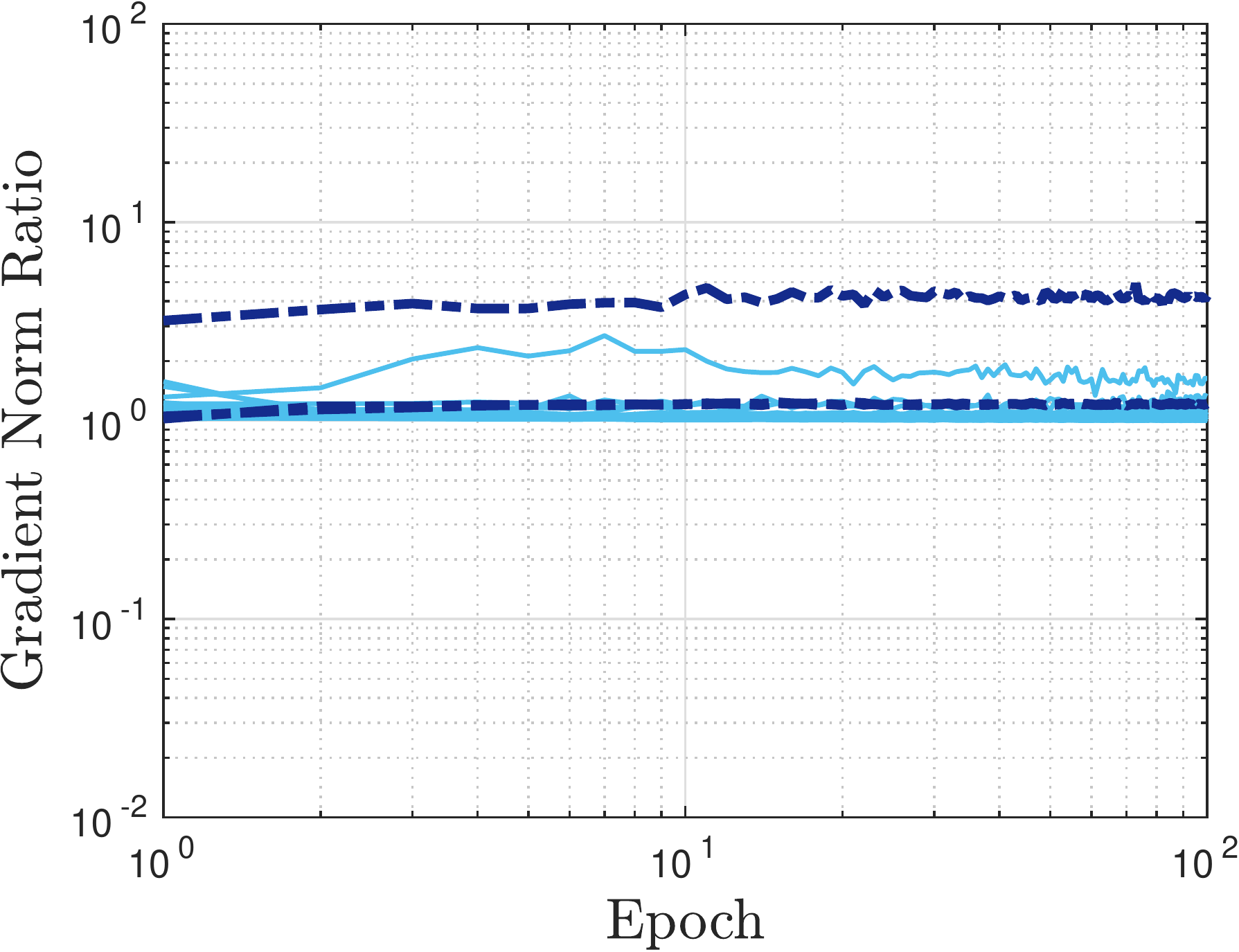}
}
\caption{ \small Training on CIFAR10. Gradient norm ratio over the first 100 epochs for transition blocks (blocks that change the dimension) and non-transition blocks (blocks that do not change the dimension). The darker color lines represent the transition blocks and the lighter color lines represent the non-transition blocks. The proposed regularization enhances the norm-preservation of the transition blocks effectively.
}
\label{fig:ratio-epoch}
\vspace{-5mm}
\end{figure*}

In the first set of experiments, the behavior of different architectures is studied as the function of network depth. To this end, the ratio of gradient norm at output to gradient norm at input, i.e., $\small \| \frac{\partial \mathcal{E}}{\partial \boldsymbol{x}_{l+1}} \|_2$ to $\| \frac{\partial \mathcal{E}}{\partial \boldsymbol{x}_{l}} \|_2$, is captured for all the residual blocks\footnote{In Plain architecture, which does not have skip connections, the gradient norm ratio is obtained at the input and output of its building blocks as depicted in Figure \ref{subfig:plain_block}.}, both transition and non-transition. Figure \ref{fig:ratio-epoch} shows the ratios for different blocks over training epochs. We ran the training for $100$ epochs, without decaying the learning rate. Plain network (Figure \ref{fig:ratio-epoch}.(g))  with 164 layers became numerically unstable and the training procedure stopped after 10 epochs.\par 

Several interesting observations can be made from this experiment:
\begin{itemize}[leftmargin=*]
\item This experiment emphasizes the fact that one needs more than careful initialization to make the network norm-preserving. Although the plain network is initially norm-preserving, the range of the gradient norm ratios becomes very large and diverges from 1, as the parameters are updated. However, ResNet and ProcResNet are able to enforce the norm-preservation during training procedure by using identity skip connection.
\item As the networks become deeper, the plain network becomes less norm preserving, which leads to numerical instability, optimization difficulty, and performance degradation. On the contrary, the non-transition blocks, the blocks with identity mapping as skip connection, of ResNet and ProcResNet become extra norm preserving. This is in line with our theoretical investigation for linear residual networks, which states that as we stack more residual blocks the network becomes extra norm-preserving. 
\item Comparing Plain83 (Figure \ref{subfig:ratio-epoch-plain83}) and Plain164 (Figure \ref{subfig:ratio-epoch-resnet164}) networks, it can be observed that most of the blocks behave fairly similar, except one transition block. Specifically, in Plain83, the gradient norm ratio of the first transition block goes up to $100$ in the first few epochs. But it eventually decreases and the network is able to converge. On the other hand, in Plain164, the gradient norm ratio of the same block becomes too large, which makes the network unable to converge. Hence, a single block is enough to make the optimization difficult and numerically unstable. This highlights the fact that it is necessary to enforce norm-preservation on all the blocks.  
\item In ResNet83 (Figure \ref{subfig:ratio-epoch-resnet83}) and ResNet164 (Figure \ref{subfig:ratio-epoch-resnet164}), it is easy to notice that only 3 transition blocks are not norm preserving. As mentioned earlier, due to multiplicative effect, the magnitude of the gradient will not be preserved because of these few blocks. 
\item The behaviors of ResNet and Plain architectures are fairly similar for depth of $20$. This was somehow expected, since it is known that the performance gain achieved by ResNet is more significant in deeper architectures \cite{He2016DeepRecognition}. However, even for depth of $20$, ProcResNet architecture is more norm preserving.
\item In ProcResNet, the only block that is less norm preserving is the first transition block, where the $3$ RGB channels are transformed into $64$ channels. This  is because, as we have shown in Figure \ref{fig:ratio_channels}, under such condition, where the number of input channels is very small, the assumption that energy of the gradient signal in the low-dimensional subspace, corresponding to the few non-zero singular values, is approximately proportional to the size of the subspace is violated with higher probability.
\item The ratios of the gradients for all networks, even the Plain network, are roughly concentrated around $1$, while training is stable. This shows that some degree of norm preservation exists in any stable network. However, as clear in the Plain network, such biases of the optimizer is not enough and we need skip connections to enforce norm preservation throughout training and to enjoy its desirable properties. Furthermore, although the transition blocks of ResNet tend to converge to be more norm preserving, our proposed modification enforces this property for all the epochs, which leads to stability and performance gain, as will be discussed shortly.
\end{itemize}

\begin{figure*}
\centering     
\subfigure[ Plain20 ]{
\label{subfig:loss-epoch-plain20}
\includegraphics[width=\tilefigwidth]{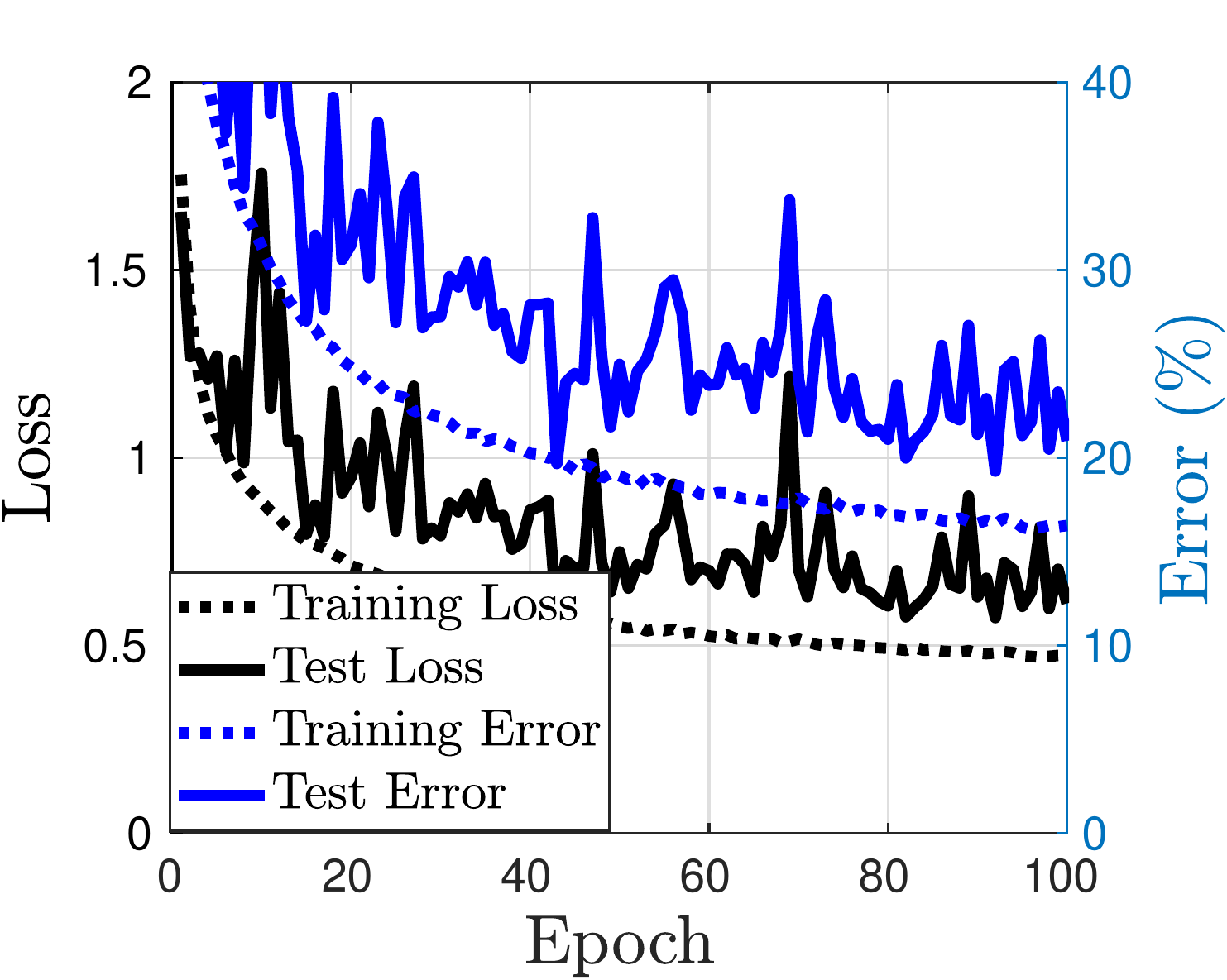}
}
\subfigure[ ResNet20]{
\label{subfig:loss-epoch-resnet20}
\includegraphics[width=\tilefigwidth]{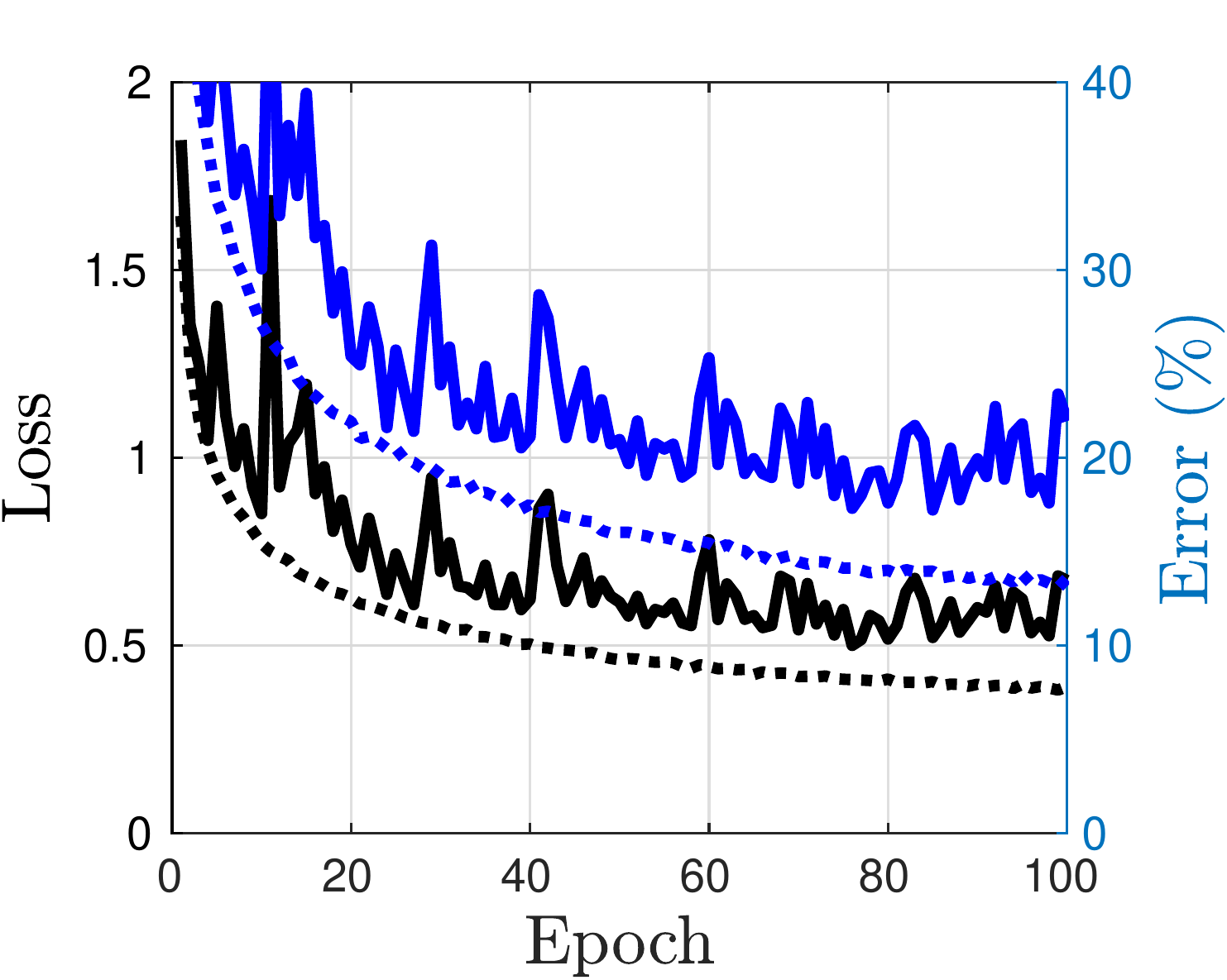}
}
\subfigure[ ProcResNet22]{
\label{subfig:loss-epoch-flatnet20}
\includegraphics[width=\tilefigwidth]{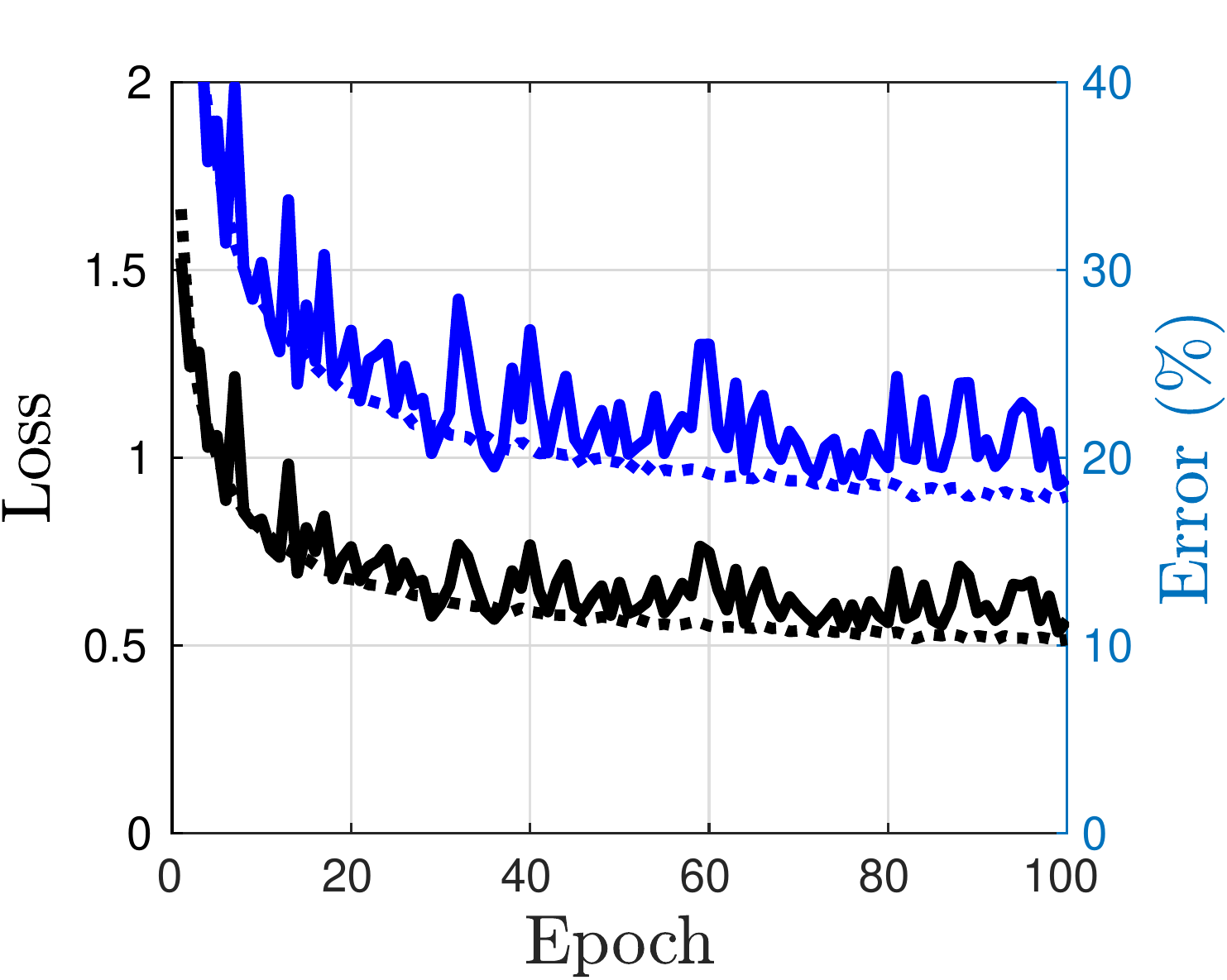}
}   
\subfigure[ Plain83]{
\label{subfig:loss-epoch-plain83}
\includegraphics[width=\tilefigwidth]{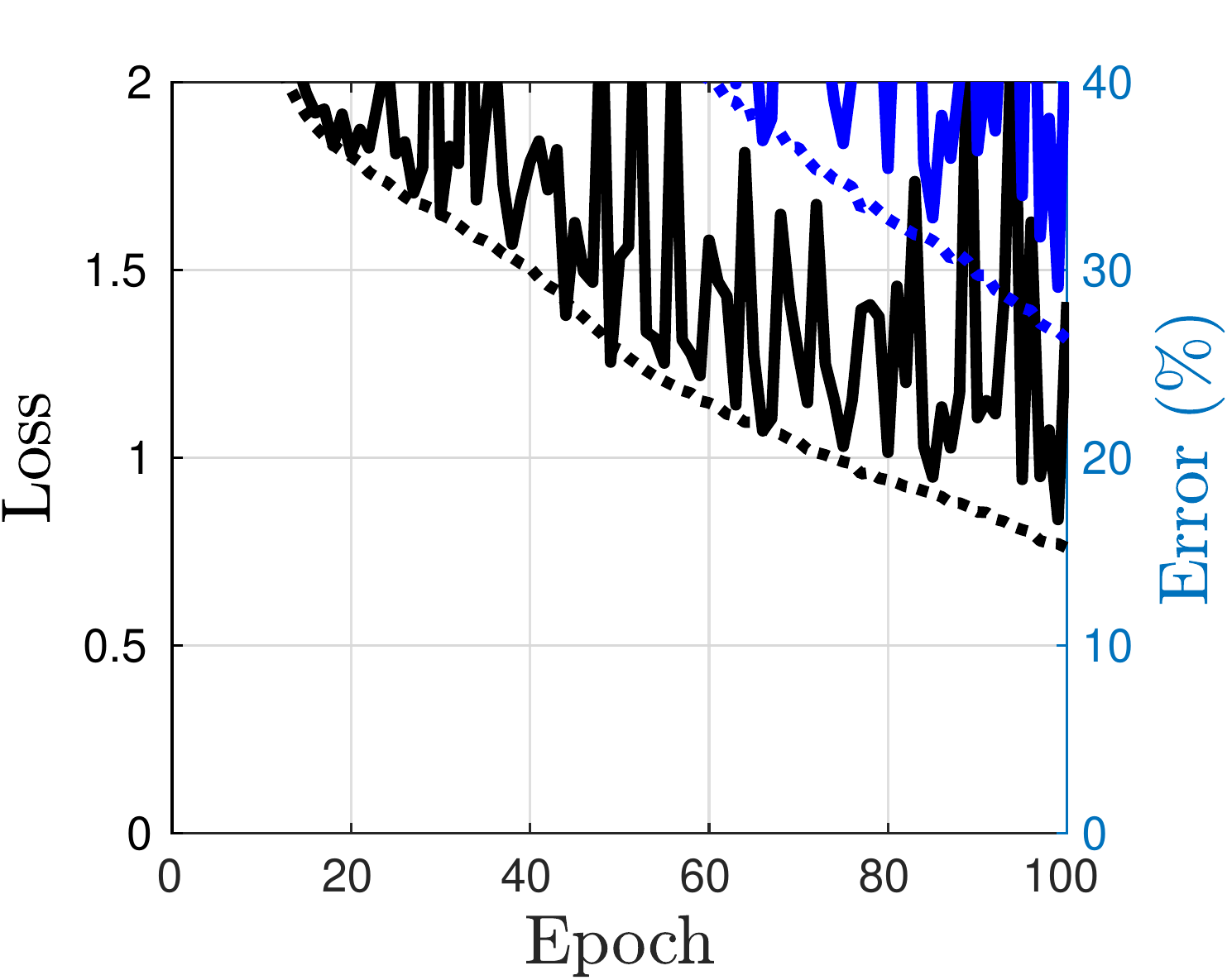}
}
\subfigure[ ResNet83]{
\label{subfig:loss-epoch-resnet83}
\includegraphics[width=\tilefigwidth]{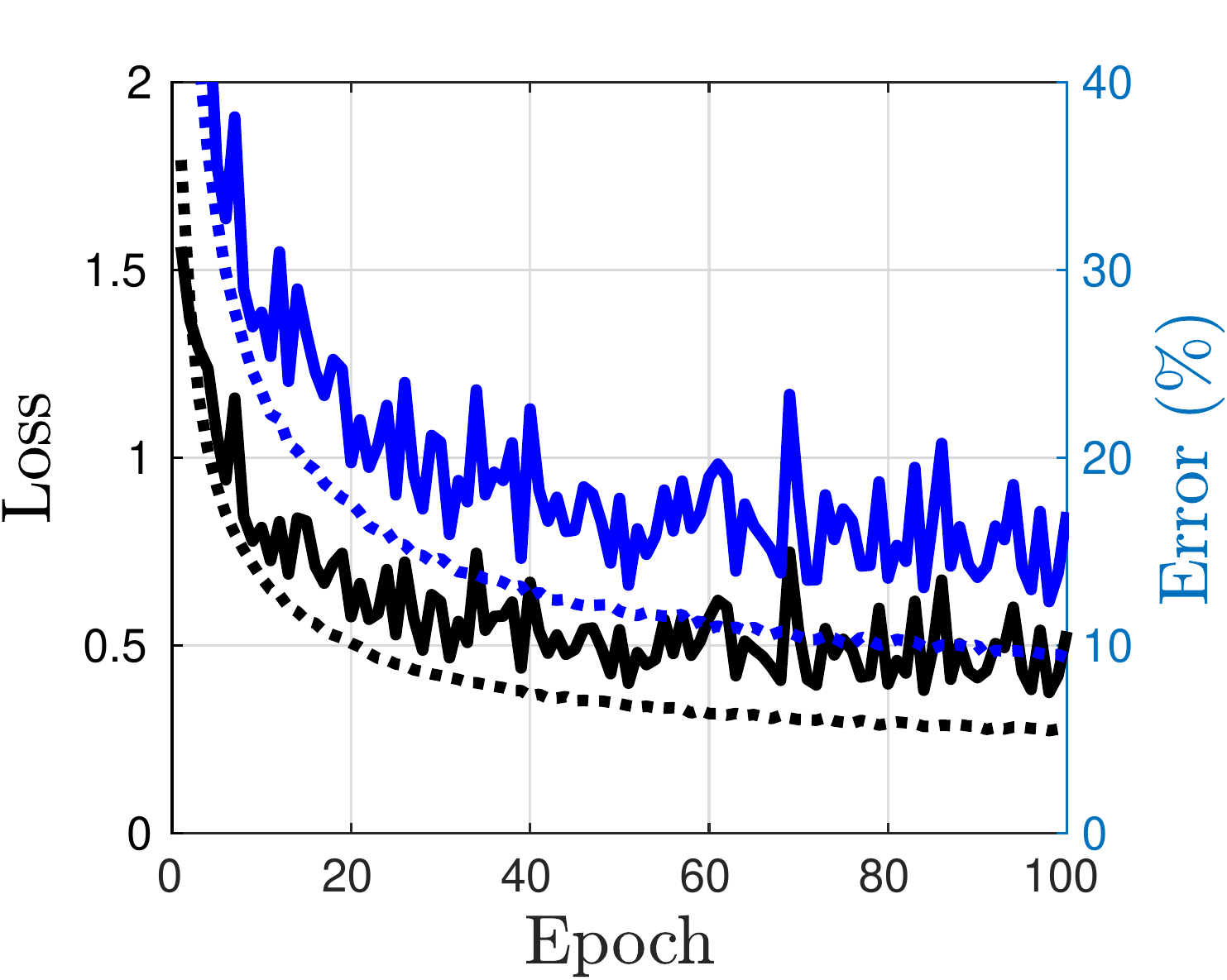}
}
\subfigure[ ProcResNet85]{
\label{subfig:loss-epoch-flatnet83}
\includegraphics[width=\tilefigwidth]{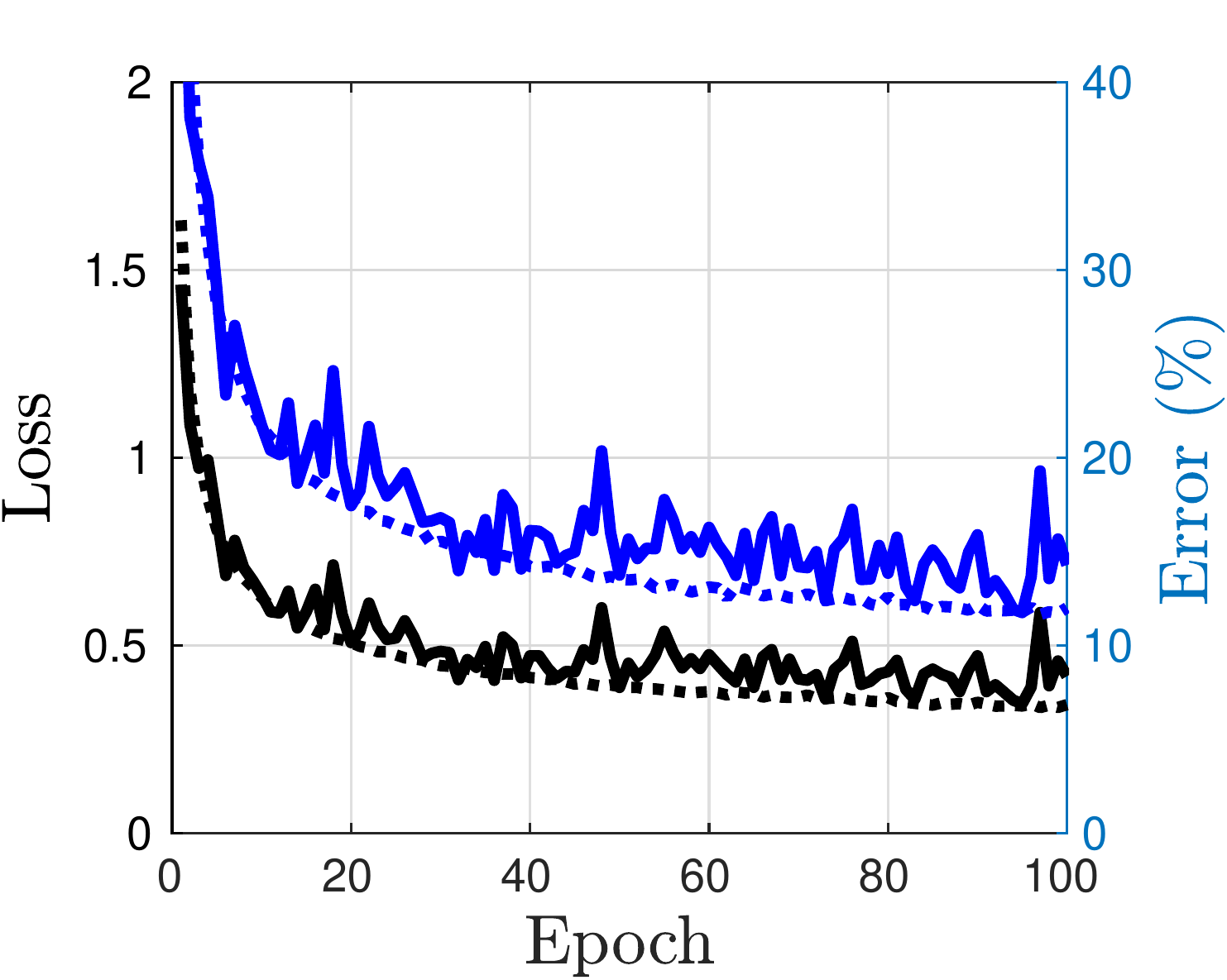}
}   
\subfigure[ Plain164]{
\label{subfig:loss-epoch-plain164}
\includegraphics[width=\tilefigwidth]{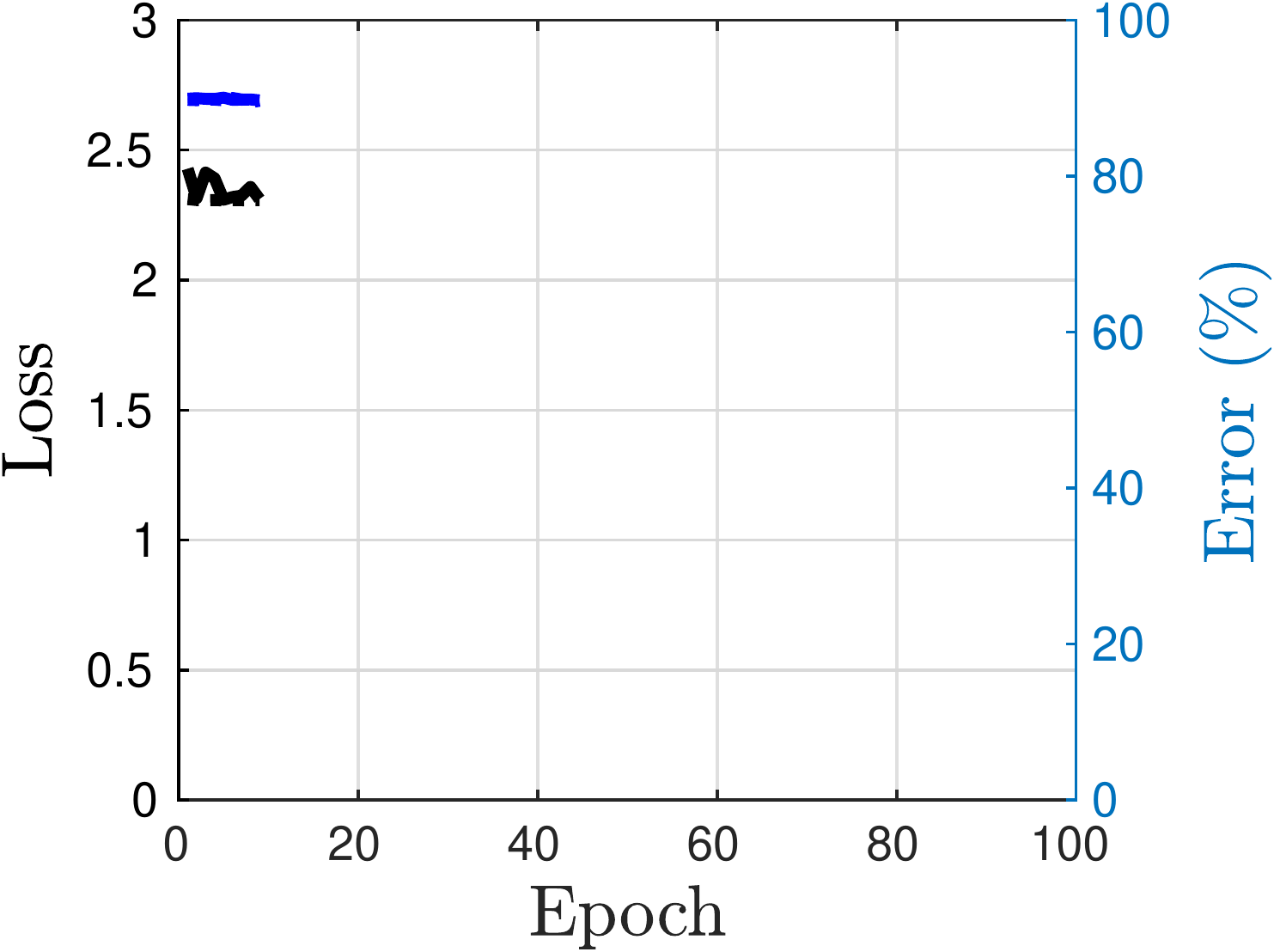}
}
\subfigure[ ResNet164]{
\label{subfig:loss-epoch-resnet164}
\includegraphics[width=\tilefigwidth]{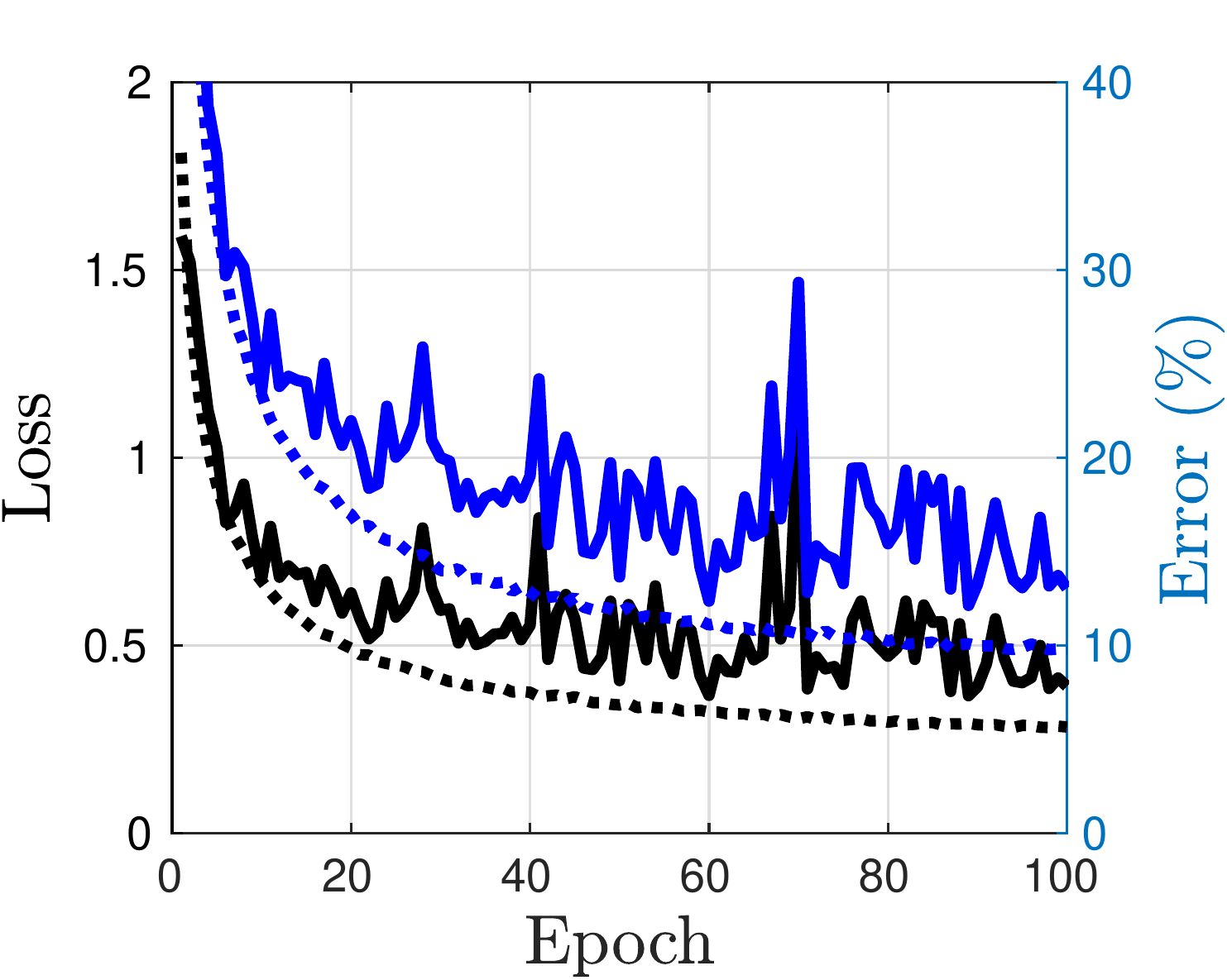}
}
\subfigure[ ProcResNet166]{
\label{subfig:loss-epoch-flatnet164}
\includegraphics[width=\tilefigwidth]{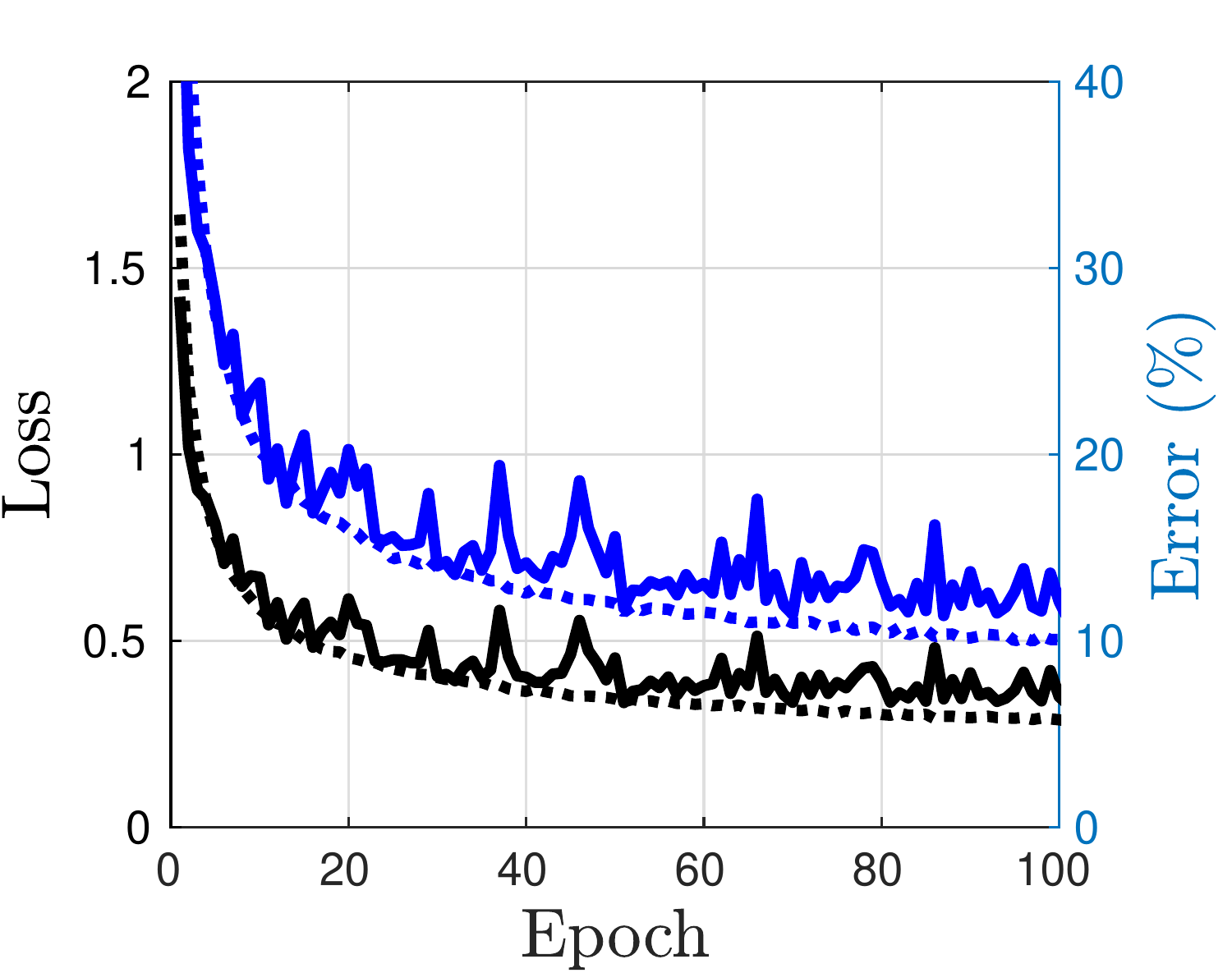}
}
\caption{ \small Loss (black lines) and error (blue lines) during training procedure on CIFAR10. Solid lines represent the test values and dotted lines represent the training values. This experiments shows how the residual connections enhance the stability of the optimization and how the proposed regularization enhances the stability even further.
}
\label{fig:loss-epoch}
\end{figure*}

This experiment both validates our theoretical arguments and clarifies some of the inner workings of ResNet architecture, and also shows the effectiveness of the proposed modifications in ProcResNet. It is evident that, as stated in Theorem \ref{thm:norm-preservation-nonlinear}, addition of identity skip connection makes the blocks increasingly extra norm-preserving, as the network becomes deeper. Furthermore, we have been able to enhance 
norm-preserving property by applying the changes proposed in Section \ref{sec:Transition Block}. 

\subsection{Optimization Stability and Learning Dynamics}
\label{subsec:exp_stability}
In the next set of experiments, numerical stability and learning dynamics of different architectures is examined. For that, loss and classification error, in both training and testing phases, are depicted in Figure \ref{fig:loss-epoch}. This experiment illustrates that how optimization stability of deep networks is improved significantly, and how it can be further improved by having norm preservation in mind during the design procedure. \par 

As depicted in Figure \ref{fig:loss-epoch}, unlike the plain network, training error and loss curves corresponding to ResNet and ProcResNet architectures are consistently decreasing as the number of layers increases, which was the main motivation behind proposing residual blocks \cite{He2016DeepRecognition}. Moreover, Figure \ref{subfig:loss-epoch-plain20} and Figure \ref{subfig:loss-epoch-plain83} show that the plain networks have a poor generalization performance. The fluctuations in testing error shows that the points along the optimization path of the plain networks do not generalize well. This issue is also present, to a lesser extent, in ResNet architecture. Comparing Figure \ref{subfig:loss-epoch-resnet164} and \ref{subfig:loss-epoch-resnet20}, we can see that the fluctuations are more apparent in deeper ResNet networks. However, in proposed ProcResNet architecture, the amplitude of the fluctuations is smaller and does not change as the depth of the network is increased. This indicates that ProcResNet architecture is taking a better path toward the optimum and has better generalization performance. \par 

To quantify this, we repeated the training $10$ times with different random seeds and measured the generalization gap, which is the difference between training and testing classification error, for the first $100$ epochs. Table \ref{table:gen_gap} shows the mean and max generalization gap, averaged over $10$ different runs. This results indicate that generalization gap of ProcResNet is smaller. Furthermore, the generalization gap fluctuates far less significantly for ProcResNet, as quantified by the difference between mean generalization gap and maximum generalization gap. 

\begin{table}
\small
\centering
\caption{\small Mean and maximum generalization gap (\%) during the first $100$ epochs of trainingon CIFA10 for different network architectures, averaged over $10$ runs.}
\label{table:gen_gap}
\begin{tabular}{|c|c|c|c|c|c|c|}
\hline
\multirow{ 2}{*}{Depth} & \multicolumn{2}{|c|}{Plain} & \multicolumn{2}{|c|}{ResNet} & \multicolumn{2}{|c|}{ProcResNet}  \\ \cline{2-7}
 & mean & max & mean & max & mean & max \\ \hline \hline
$20$ & 6.7 & 20.0 & 5.5 & 23.1 & \bf{2.3} & \bf{8.3} \\ \hline
$83$ & 7.5 & 30.1 & 5.1 & 12.5 & \bf{2.0} & \bf{7.7}\\ \hline
$164$ & -  & -    & 5.2 & 18.7 & \bf{3.3} & \bf{8.7} \\ \hline
\end{tabular}
\end{table}

The implication of this is that  by modifying only a few blocks in an extremely deep network, it is possible to make the network more stable and improve the learning dynamics. This emphasizes the utmost importance of norm-preservation of all blocks in avoiding optimization difficulties of very deep networks. Moreover, this sheds light on the reasons why architectures using residual blocks, or identity skip connection in general, perform so well and are easier to optimize. \par 

\subsection{Classification Performance}
\label{subsec:exp_accuracy}
In this section, we show the impact of the proposed norm-preserving transition blocks on the classification performance of ResNet. Table \ref{table:error} compares the performance of ResNet and its EraseReLU version, as proposed in \cite{Dong2017Eraserelu:Networks}, with and without the proposed transition blocks. 
The results for standard ResNet are the best results reported by \cite{He2016IdentityNetworks} and \cite{Dong2017Eraserelu:Networks} and the results of ProcResNet are obtained by making the proposed changes to standard ResNet implementation. \par 

Table \ref{table:error} shows that the proposed network performs better than the standard ResNet. This performance gain comes with a slight increase the number of parameters (under 1\%) and by changing only $3$ blocks. The total number of residual blocks for ResNet164 and ResNet1001 are  $54$ and $333$, respectively. 
Furthermore, Figure \ref{fig:error-depth} compares the parameter efficiency of ResNet and ProcResNet architectures. The results indicate that the proposed modification can improve the parameter efficiency significantly. For example, ProcResNet274 (with 2.82M parameter) slightly outperforms ResNet1001 (with 10.32M parameters). This translates into about 4x reduction in the number of parameters to achieve the same classification accuracy.
This illustrates that we are able to improve the performance by changing a tiny portion of the network and emphasizes the importance of norm-preservation in the performance of neural networks. \par

\begin{figure}
\centering 
\includegraphics[width=\figurewidth]{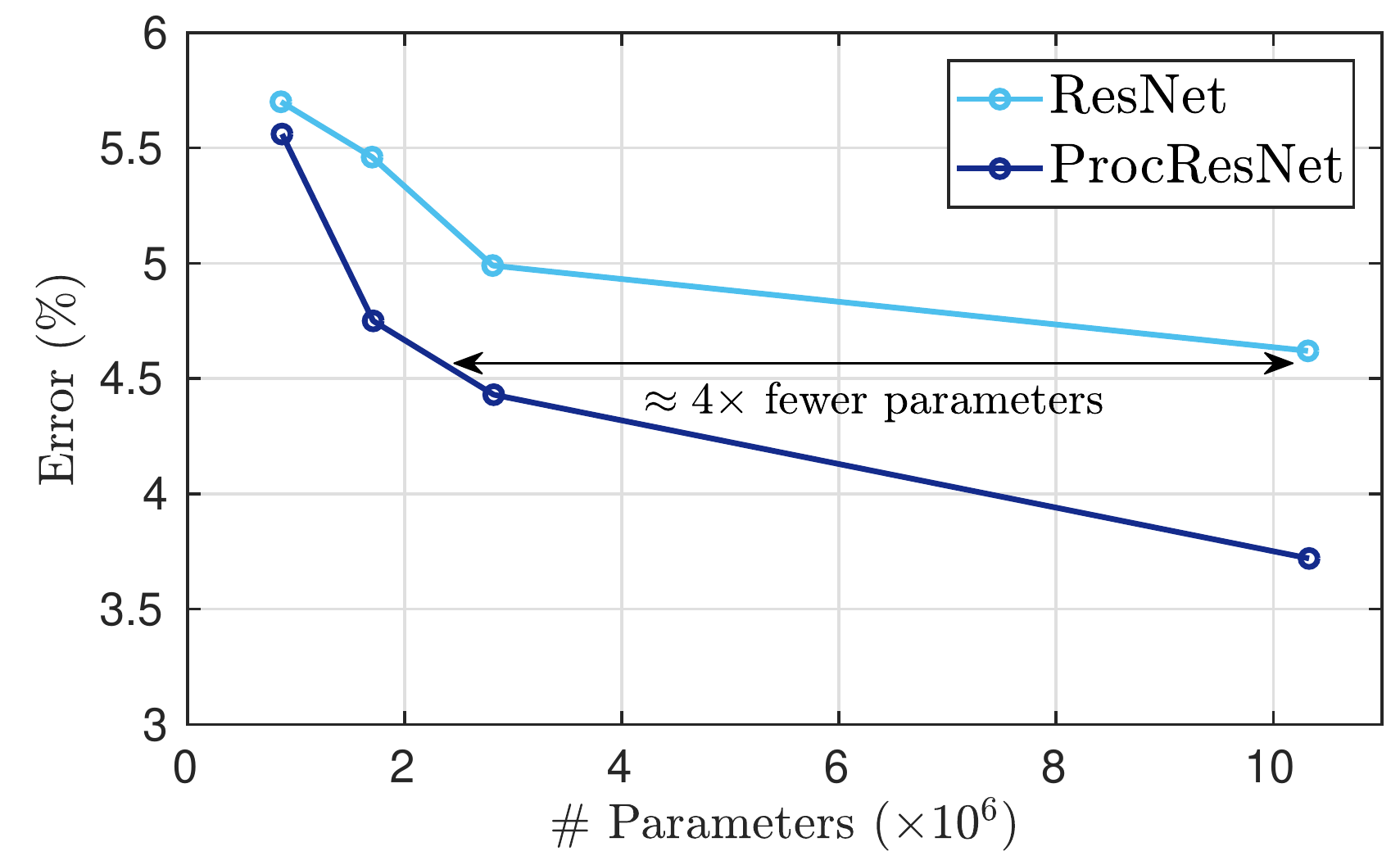}
\caption{ \small  Comparison of the parameter efficiency on CIFAR10 between ResNet and ProcResNet.}
\label{fig:error-depth}    
\end{figure}

Finally, Table \ref{table:ablation} investigates the impact of changing the architecture, i.e., moving the convolution layer from the skip connection to before the skip connection, and performing the proposed regularization, separately. Each of these design components have positive impact on the performance of the network, as both of them enhance the norm preservation of the transition block, which further highlights the impact of norm preservation on the performance of the network.

\begin{table*}
\setlength\extrarowheight{2pt}
\centering
\caption{\small  Performance of different  methods on CIFAR-10 and CIFAR-100 using moderate data augmentation (flip/translation). The modified transition blocks in ProcResNet can improve the accuracy of ResNet significantly.}
\label{table:error}
\begin{tabular}{|c|c|c|c|c|c|}
\hline
\multirow{ 2}{*}{Architecture} & \multirow{ 2}{*}{Setting} & \multirow{ 2}{*}{\# Params}  & \multirow{ 2}{*}{Depth}    & \multicolumn{2}{|c|}{Error (\%)} \\ \cline{5-6}
 & & & & CIFAR10 & CIFAR100 \\ \hline \hline
 
\multirow{ 4}{*}{ResNet \cite{He2016IdentityNetworks}} & \multirow{ 2}{*}{pre-activation} & 1.71M & 164 & 5.46 & 24.33 \\ \cline{3-6}
  &  & 10.32M & 1001 & 4.62  & 22.71  \\ \cline{2-6}
  
& \multirow{ 2}{*}{ErasedReLU\cite{Dong2017Eraserelu:Networks}} & 1.70M & 164 & 4.65 & 22.41 \\ \cline{3-6}
  &  & 10.32M & 1001 & 4.10  & 20.63  \\ \hline 

\multirow{ 4}{*}{ProcResNet} & \multirow{ 2}{*}{pre-activation} & 1.72M & 166 & 4.75 & 22.61 \\ \cline{3-6}
&  & 10.33M & 1003 & 3.72 & 19.99  \\ \cline{2-6}

& \multirow{ 2}{*}{ErasedReLU\cite{Dong2017Eraserelu:Networks}} & 1.72M & 166 & 4.53 & 21.91 \\ \cline{3-6}
&  & 10.33M & 1003 & \bf 3.42 & \bf 18.12  \\ \hline
\end{tabular}
\end{table*}

\begin{table}
\centering
\caption{\small Ablation study on ResNet with 164 layers on CIFAR100.}
\label{table:ablation}
\begin{tabular}{|c|c|c|}
\hline
Transition Block & Projection & Error (\%)\\ \hline
Original & No & 24.33 \\ \hline
Modified & No & 23.06 \\ \hline
Modified & Yes & 22.61 \\ \hline
\end{tabular}
\end{table}

\section{Conclusions}
\label{sec:conclusions}

This paper theoretically analyzes building blocks of residual networks  and demonstrates that adding identity skip connection makes the residual blocks norm-preserving. Furthermore, the norm-preservation is enforced during the training procedure, which makes the optimization stable and improves the performance. This is in contrast to initialization techniques, such as \cite{Glorot2010UnderstandingNetworks}, which ensure norm-preservation only at the beginning of the training. Our experiments validate our theoretical investigation by showing that
\begin{enumerate*}[label=(\roman*)]
\item identity skip connection results in norm preservation, 
\item residual blocks become extra norm-preserving as the network becomes deeper, and
\item the training can become more stable through enhancing the norm preservation of the network.
\end{enumerate*}
Our proposed modification of ResNet, Procrustes ResNet, enforces norm-preservation on the transition blocks of the network and is able to achieve better optimization stability and performance. For that we propose an efficient regularization technique to set the nonzero singular values of the convolution operator, without performing singular value decomposition.
Our findings can be seen as design guidelines for very deep architectures. By having norm-preservation in mind, we will be able to train extremely deep networks and alleviate the optimization difficulties of such networks. \par
 
\section{Acknowledgements}
This research is based upon work supported in parts by the National Science Foundation under Grants No. 1741431 and CCF-1718195 and the Office of the Director of National Intelligence (ODNI), Intelligence Advanced Research Projects Activity (IARPA), via IARPA R\&D Contract No. D17PC00345. The views, findings, opinions, and conclusions or recommendations contained herein are those of the authors and should not be interpreted as necessarily representing the official policies or endorsements, either expressed or implied, of the NSF, ODNI, IARPA, or the U.S. Government. The U.S. Government is authorized to reproduce and distribute reprints for Governmental purposes notwithstanding any copyright annotation thereon.

\balance
\bibliography {Mendeley}
\bibliographystyle{ieeetr}

\begin{IEEEbiography} [{\includegraphics[width=1in,height=1.25in,clip,keepaspectratio]{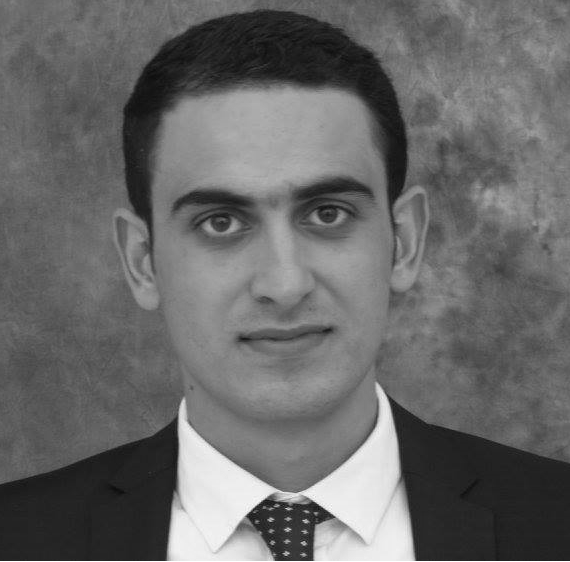}}]
{Alireza Zaeemzadeh}
(S`11) received the B.S. degree in electrical engineering from the University of Tehran, Tehran, Iran, in 2014. He is currently working toward the Ph.D. degree in electrical engineering at the University of Central Florida. His current research interests lie in the areas of machine learning, linear algebra, and optimization. Alireza`s awards and honors include University of Central Florida Multidisciplinary Doctoral Fellowship and Graduate Dean's Fellowship.
\end{IEEEbiography}

\begin{IEEEbiography} [{\includegraphics[width=1in,height=1.25in,clip,keepaspectratio]{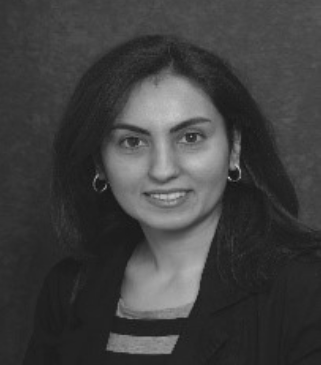}}]
{Nazanin Rahnavard}
 (S`97-M`10-SM'19) received   her Ph.D.  in  the  School  of  Electrical  and  Computer Engineering at the Georgia Institute of Technology, Atlanta, in 2007. She is currently an Associate Professor in the Department of Electrical and Computer Engineering  at  the  University  of  Central  Florida, Orlando,  Florida.  Dr.  Rahnavard  is  the  recipient of  NSF  CAREER  award  in  2011 and 2020 UCF's College of Engineering and Computer Science Excellence in Research Award.  She  has  interest and  expertise  in  a  variety  of  research  topics  in  the communications, networking, signal processing, and machine learning areas. She   serves   on   the   editorial   board   of   the Elsevier  Journal  on  Computer  Networks  (COMNET)  and  on  the  Technical Program Committee of several prestigious international conferences.
\end{IEEEbiography}

\begin{IEEEbiography}[{\includegraphics[width=1in,height=1.25in,clip,keepaspectratio]{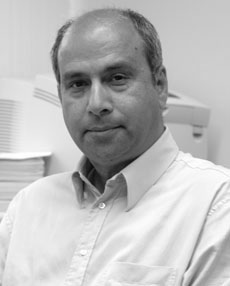}}]{Mubarak Shah}, the UCF Trustee chair professor, is the founding director of the Center for Research in Computer Vision at the University of Central Florida (UCF).  He is a fellow of the NAI, IEEE, AAAS, IAPR, and SPIE. He is an editor of an international book series on video computing, was editor-in-chief of Machine Vision and Applications journal, and an associate editor of ACM Computing Surveys journal. He was the program cochair of CVPR 2008, an associate editor of the IEEE T-PAMI, and a guest editor of the special issue of the International Journal of Computer Vision on Video Computing. His research interests include video surveillance, visual tracking, human activity recognition, visual analysis of crowded scenes, video registration, UAV video analysis, and so on. He has served  as an ACM distinguished speaker and IEEE distinguished visitor speaker.  He is a recipient of ACM SIGMM Technical Achievement award; IEEE Outstanding Engineering Educator Award; Harris Corporation Engineering Achievement Award; an honorable mention for the ICCV 2005 Where Am I? Challenge Problem; 2013 NGA Best Research Poster Presentation; 2nd place in Grand Challenge at the ACM Multimedia 2013 conference; and runner up for the best paper award in ACM Multimedia Conference in 2005 and 2010. At UCF he has received Pegasus Professor Award; University Distinguished Research Award; Faculty Excellence in Mentoring Doctoral Students; Scholarship of Teaching and Learning award; Teaching Incentive Program award; Research Incentive Award.
\end{IEEEbiography}


\appendices
\section{Proofs}

\subsection{Proof of Theorem \ref{thm:norm-preservation-nonlinear}}
\label{sec:app_nonlin_proof}
 For a   cost function $\mathcal{E}(.)$  and  the Jacobian $\boldsymbol{J}$ of   $\boldsymbol{x}_{l+1}$ with respect to $\boldsymbol{x}_{l}$, applying chain rule, following is true:
$$
\begin{aligned}
& \frac{\partial \mathcal{E}}{\partial \boldsymbol{x}_{l}} = \boldsymbol{J} \frac{\partial \mathcal{E}}{\partial \boldsymbol{x}_{l+1}}, \\
& \boldsymbol{J}= \frac{\partial \boldsymbol{x}_{l+1}}{\partial \boldsymbol{x}_{l}} =  \boldsymbol{I} + DF_{l}(\boldsymbol{x}_{l}),
\end{aligned}
$$
where $D$ is the differential operator and for any $\boldsymbol{v}$ with bounded norm we have:
$$
DF_{l}(\boldsymbol{x}_{l})\boldsymbol{v} = \lim_{t \rightarrow 0^+} \frac{F_{l}(\boldsymbol{x}_{l} + t\boldsymbol{v}) - F_{l}(\boldsymbol{x}_{l})}{t}
$$
To prove Theorem \ref{thm:norm-preservation-nonlinear}, we first state a lemma.

\begin{lem}
\label{lemma:sigma-bound}
For any non-singular matrix $\boldsymbol{I} + \boldsymbol{M}$, we have: 
$$
\small
1 - \sigma_{max}(\boldsymbol{M}) \leq \sigma_{min}(\boldsymbol{I} + \boldsymbol{M})  \leq \sigma_{max}(\boldsymbol{I} + \boldsymbol{M}) \leq 1 + \sigma_{max}(\boldsymbol{M}), 
$$ 
where $\sigma_{min}(\boldsymbol{M})$ and $\sigma_{min}(\boldsymbol{M})$ represent the minimum and maximum singular values of $\boldsymbol{M}$, respectively. 
\end{lem}
\begin{proof}
Since $ \sigma_{min}(\boldsymbol{I} + \boldsymbol{M}) > 0$, the lower bound is trivial for $\sigma_{max}(\boldsymbol{M}) \geq 1$. For $\sigma_{max}(\boldsymbol{M}) < 1$, it is known that $| \lambda_{max}(\boldsymbol{M}) | < 1$, where $\lambda_{max}(\boldsymbol{M})$ is the maximum eigenvalue of $\boldsymbol{M}$ \cite{Derzko1965BoundsMatrix}. Thus, we can show that:
$$
\begin{aligned}
& \sigma_{min}(\boldsymbol{I} + \boldsymbol{M}) 
&& =  (\sigma_{max}((\boldsymbol{I} + \boldsymbol{M})^{-1}))^{-1} \\
&
&& = \| (\boldsymbol{I} + \boldsymbol{M})^{-1} \|_2^{-1}  \\
&
&& \overset{(a)}{=} \| \sum_{k=0}^{\infty}(-1)^{k}\boldsymbol{M}^k \|_2^{-1}\\
&
&& \geq (\sum_{k=0}^{\infty}\|(-1)^{k}\boldsymbol{M}^k  \|_2)^{-1} \geq (\sum_{k=0}^{\infty}\|\boldsymbol{M} \|_2^k)^{-1}\\
&
&&   = (\frac{1}{1 - \|\boldsymbol{M} \|_2})^{-1} = 1 - \sigma_{max}( \boldsymbol{M}).\\
\end{aligned}
$$
Identity (a) is known as Neuman series of a matrix, which holds when $| \lambda_{max}(\boldsymbol{M}) | < 1$ and $||.||_2$ represents the $l_2$-norm of a matrix.

The upper bound is easier to show. Due to triangle inequality:
$$
 \sigma_{max}(\boldsymbol{I} + \boldsymbol{M}) = \| \boldsymbol{I} + \boldsymbol{M} \|_2 \leq \| \boldsymbol{I}\|_2 + \| \boldsymbol{M} \|_2 = 1 + \sigma_{max}(\boldsymbol{M}).
$$

\end{proof}
Thus, knowing that
$$
\sigma_{min}(\boldsymbol{J}) \| \frac{\partial \mathcal{E}}{\partial \boldsymbol{x}_{l+1}} \|_2 \leq \| \boldsymbol{J} \frac{\partial \mathcal{E}}{\partial \boldsymbol{x}_{l+1}} \|_2 \leq \sigma_{max}(\boldsymbol{J})\| \frac{\partial \mathcal{E}}{\partial \boldsymbol{x}_{l+1}} \|_2,
$$
using Lemma \ref{lemma:sigma-bound}, we conclude that 

$$
(1-\delta') \| \frac{\partial \mathcal{E}}{\partial \boldsymbol{x}_{l+1}} \|_2 \leq \| \frac{\partial \mathcal{E}}{\partial \boldsymbol{x}_{l}} \|_2 \leq (1+\delta')\| \frac{\partial \mathcal{E}}{\partial \boldsymbol{x}_{l+1}} \|_2.
$$
where, $\delta'=\sigma_{max}(DF_{l}(\boldsymbol{x}_{l}))$. Furthermore, we have:
$$
\small
\begin{aligned}
& \sigma_{max}(DF_{l}(\boldsymbol{x}_{l}))
&& = \sup_{\boldsymbol{v} }  \frac{\|DF_{l}(\boldsymbol{x}_{l}) \boldsymbol{v} \|_2}{\|\boldsymbol{v} \|_2}
\\ & && 
= \sup_{\boldsymbol{v} } \frac{1}{\|\boldsymbol{v} \|_2}\|  \lim_{t \rightarrow 0^+} \frac{F_{l}(\boldsymbol{x}_{l} + t\boldsymbol{v}) - F_{l}(\boldsymbol{x}_{l})}{t}\|_2
\\& &&
=  \lim_{t \rightarrow 0^+}
\sup_{\boldsymbol{v} }
\frac{1}{\|\boldsymbol{v} \|_2} \| \frac{F_{l}(\boldsymbol{x}_{l} + t\boldsymbol{v}) - F_{l}(\boldsymbol{x}_{l})}{t}\|_2
\\& &&
=  \lim_{t \rightarrow 0^+}
\sup_{\boldsymbol{v} }  \frac{\|F_{l}(\boldsymbol{x}_{l} + t\boldsymbol{v}) - F_{l}(\boldsymbol{x}_{l})\|_2}{t\|\boldsymbol{v} \|_2}
\\ & &&
\leq \| F_{l} \|_L,
\end{aligned}
$$
where $\| f \|_L$ is the Lipschitz seminorm of function $ f$ and is defined as
$$
\| f \|_L := \sup_{\boldsymbol{x} \neq \boldsymbol{y}} \frac{\|f(\boldsymbol{x}) - f(\boldsymbol{y})\|_2}{\|\boldsymbol{x} - \boldsymbol{y}\|_2}.
$$

To conclude the proof, we use the following lemma:
\begin{lem}
(Theorem 1 in \cite{Bartlett2018RepresentingOptimization}) Suppose we want to represent a nonlinear mapping $F: \mathbb{R}^N \rightarrow \mathbb{R}^N$, satisfying Assumption \ref{assump:nonlinear_map}, with a sequence of $L$ non-linear residual blocks of form $\boldsymbol{x}_{l+1} = \boldsymbol{x}_{l} + F_l(\boldsymbol{x}_{l})$. There exists a solution such that for all residual blocks we have $\| F_{l} \|_L \leq c\frac{\log(2L)}{L}$.
\end{lem}

Therefore, $\delta' = \sigma_{max}(DF_{l}(\boldsymbol{x}_{l})) \leq \| F_{l} \|_L \leq c\frac{\log(2L)}{L} = \delta $, which concludes the proof.

\subsection{Proof of Theorem \ref{thm:linear-norm-preservation}}
\label{sec:app_lin_proof}
In the classical back-propagation equation, for a cost function $\mathcal{E}(.)$  and  the Jacobian $\boldsymbol{J}$ of   $\boldsymbol{x}_{l+1}$ with respect to $\boldsymbol{x}_{l}$, applying chain rule, following is true:
\begin{equation}
\label{eq:jacobian}
\begin{aligned}
& \frac{\partial \mathcal{E}}{\partial \boldsymbol{x}_{l}} = \boldsymbol{J} \frac{\partial \mathcal{E}}{\partial \boldsymbol{x}_{l+1}}, \\
& \boldsymbol{J}= \frac{\partial \boldsymbol{x}_{l+1}}{\partial \boldsymbol{x}_{l}} =  \boldsymbol{I} + \boldsymbol{W}_{l}^T,
\end{aligned}
\end{equation}

To prove the theorem, using Lemma \ref{lemma:sigma-bound} and knowing that
$$
\sigma_{min}(\boldsymbol{J}) \| \frac{\partial \mathcal{E}}{\partial \boldsymbol{x}_{l+1}} \|_2 \leq \| \boldsymbol{J} \frac{\partial \mathcal{E}}{\partial \boldsymbol{x}_{l+1}} \|_2 \leq \sigma_{max}(\boldsymbol{J})\| \frac{\partial \mathcal{E}}{\partial \boldsymbol{x}_{l+1}} \|_2,
$$
we conclude that 
$$
(1-\delta') \| \frac{\partial \mathcal{E}}{\partial \boldsymbol{x}_{l+1}} \|_2 \leq \| \frac{\partial \mathcal{E}}{\partial \boldsymbol{x}_{l}} \|_2 \leq (1+\delta')\| \frac{\partial \mathcal{E}}{\partial \boldsymbol{x}_{l+1}} \|_2.
$$
where, $\delta'=\sigma_{max}(\boldsymbol{W}_{l})$. To conclude the proof, we use the following lemma.
\begin{lem}
\label{lemma:norm-res-bound}
(Theorem 2.1 in \cite{Hardt2017IdentityLearning}) Suppose $L \geq 3\gamma$. Then, there exists a global optimum for $\mathcal{E}(\mathcal{W})$, such that we have
$$
\sigma_{max}(\boldsymbol{W}_{l}) \leq \frac{2(\sqrt[]{\pi} + \sqrt[]{3\gamma})^2}{L}, \forall l = 1,2,\dots,L,
$$
where $\gamma$ is $\max (|\log \sigma_{max}(\boldsymbol{R})|,|\log \sigma_{min}(\boldsymbol{R})|)$.
\end{lem}

Using the results from this lemma and setting $\delta = \frac{2(\sqrt[]{\pi} + \sqrt[]{3\gamma})^2}{L}$, 
Theorem \ref{thm:linear-norm-preservation} follows immediately.

\subsection{Proof for Corollary \ref{cor:norm-preservation-weights}}
\label{app:proof-nonlinear}
Here, Jacobian matrix is $\boldsymbol{J} = \boldsymbol{I} + \boldsymbol{F}'{\boldsymbol{W}^{(1)}}^T\boldsymbol{F}'{\boldsymbol{W}_{l}^{(2)}}^T$, where $\boldsymbol{F}'$ is the Jacobian of $\boldsymbol{\rho}(.)$  with respect to its input. Since we know that $0 \leq \frac{\partial \rho_n(\boldsymbol{x})}{\partial x_{n'}} \leq c_{\rho}, \forall n = n'$ and $\frac{\partial \rho_n(\boldsymbol{x})}{\partial x_{n'}} = 0, \forall n \neq n'$, we have $\| \boldsymbol{F} \|_2 \leq c_{\rho}$. Therefore:
$$
\begin{aligned}
\small
& \| \boldsymbol{F}'{\boldsymbol{W}_{l}^{(1)}}^T\boldsymbol{F}'{\boldsymbol{W}_{l}^{(2)}}^T\|_2 
&& \leq \| \boldsymbol{F}'\|_2 \| {\boldsymbol{W}_{l}^{(1)}}^T\|_2 \| \boldsymbol{F}'\|_2  \| {\boldsymbol{W}_{l}^{(1)}}^T\|_2 \\
&
&& \leq c_{\rho}^2 \| \boldsymbol{W}_{l}^{(1)}\|_2 \| \boldsymbol{W}_{l}^{(2)}\|_2
\end{aligned}
$$
and using Lemma \ref{lemma:sigma-bound} and setting $\delta = c_{\rho}^2 \| \boldsymbol{W}_{l}^{(1)}\|_2 \| \boldsymbol{W}_{l}^{(2)}\|_2$, Corollary \ref{cor:norm-preservation-weights} follows immediately. 
\end{document}